%% file: ecai-sample-and-instructions.tex
\documentclass{ecai} 

\usepackage{latexsym}
\usepackage{amssymb}
\usepackage{amsmath}
\usepackage{booktabs}
\usepackage{enumitem}
\usepackage{mathtools}
\usepackage{amsthm}
\usepackage{algorithm}
\usepackage{algorithmic}
\usepackage[utf8]{inputenc} 
\usepackage[T1]{fontenc}    
\usepackage{url}            
\usepackage{booktabs}       
\usepackage{amsfonts}       
\usepackage{nicefrac}       
\usepackage{microtype}      
\usepackage{xcolor}         
\usepackage{microtype}
\usepackage{graphicx}
\usepackage{subfigure}
\usepackage{booktabs} 
\usepackage{capt-of}
\usepackage{xspace}
\usepackage{microtype}
\usepackage{color}
\usepackage{url}
\usepackage{lipsum}
\usepackage{fancyhdr}
\usepackage{blindtext}
\usepackage{babel}
\usepackage{nameref}
\input{math_commands.tex}

\newtheorem{theorem}{Theorem}
\newtheorem{lemma}[theorem]{Lemma}
\newtheorem{assumption}[theorem]{Assumption}

\newcommand{\BibTeX}{B\kern-.05em{\sc i\kern-.025em b}\kern-.08em\TeX}

\begin{document}


\begin{frontmatter}


\paperid{1624} 


\title{Adjacent Leader Decentralized \\
Stochastic Gradient Descent}


\author[A]{\fnms{Haoze}~\snm{He}\footnote{Equal contribution.}}
\author[B]{\fnms{Jing}~\snm{Wang}\footnotemark}
\author[C]{\fnms{Anna}~\snm{Choromanska}} 

\address[A,B,C]{New York University, Tandon School of Engineering}


\begin{abstract}
This work focuses on the decentralized deep learning optimization framework. We propose Adjacent Leader Decentralized Gradient Descent (AL-DSGD), for improving final model performance, accelerating convergence, and reducing the communication overhead of decentralized deep learning optimizers. AL-DSGD relies on two main ideas. Firstly, to increase the influence of
the strongest learners on the learning system it assigns weights to different neighbor workers according to both their performance and the degree when averaging among them, and it applies a corrective force on the workers dictated by both the currently best-performing neighbor and the neighbor with the maximal degree. 
Secondly, to alleviate the problem of the deterioration of the convergence speed and performance of the nodes with lower degrees, AL-DSGD relies on dynamic communication graphs, which effectively allows the workers to communicate with more nodes while keeping the degrees of the nodes low. 
Experiments demonstrate that AL-DSGD accelerates the convergence of the decentralized state-of-the-art techniques
and improves their test performance especially in the communication constrained environments. We also theoretically prove the convergence of the proposed scheme.
Finally, we release to the community a highly general and concise PyTorch-based library for distributed training of deep learning models that supports easy implementation of any distributed deep learning approach ((a)synchronous, (de)centralized). 
\end{abstract}

\end{frontmatter}


\section{Introduction}

Stochastic gradient descent (SGD) is the skeleton of most state-of-the-art (SOTA) machine learning algorithms. The stability and convergence rate of classical SGD, which runs serially at a single node, has been well studied~\citep{dekel2012optimal, ghadimi2013stochastic}. However, recently, the dramatical increase in the size of deep learning models~\citep{devlin2018bert, brown2020language, radford2021learning, yuan2021florence}, the amount of computations and the size of the data sets~\citep{deng2009imagenet, krizhevsky2009learning} make it challenging to train the model on a single machine. To efficiently process these quantities of data with deep learning models, distributed SGD, which parallelizes training across multiple workers, has been successfully employed. However, centralized distributed SGD with a parameter server~\citep{dean2013tail, li2014communication, cui2014exploiting, dutta2016short, dean2012large} suffers from the communication bottleneck problem when the framework either has a large number of workers or low network bandwidth~\citep{liu2020decentralized, lian2018asynchronous, lian2017can,wang2021cooperative}. 

To overcome the communication bottleneck, decentralized frameworks have been proposed. Moving from centralized to decentralized learning system is a highly non-trivial task. Many popular distributed deep learning schemes, including EASGD~\citep{zhang2015deep}, LSGD~\citep{teng2019leader}, as well as PyTorch embedded scheme called DataParallel~\citep{li2020pytorch} are only compatible with the centralized framework. In a decentralized setting, there exists no notion of the central parameter server and the communication between workers is done over the communication network with a certain topology. The decentralized approach is frequently employed and beneficial for training in various settings, including sensor networks, multi-agent systems, and federated learning on edge devices. Most of the decentralized SGD algorithms communicate over the network with pre-defined topology and utilize parameter averaging instead of gradient updates during the communication step. This is the case for the SOTA approach called decentralized parallel SGD (D-PSGD) \cite{lian2017can}, that allows each worker to send a copy of its model to its adjacent workers at every iteration, and its variants~\cite{lian2017can, lian2018asynchronous, wang2022matcha, koloskova2020unified}. These methods satisfy convergence guarantees in terms of iterations or communication rounds~\citep{duchi2011dual, jakovetic2018convergence, nedic2009distributed,scaman2018optimal,towfic2016excess,yuan2016convergence, zeng2018nonconvex, lian2017can,lian2018asynchronous}. Since each worker only needs to compute an average with its neighbors, they reduce communication complexity compared to centralized methods~\citep{lian2018asynchronous, blot2016gossip, jin2016scale, lian2017can, wang2022matcha}. (Neighbor workers refer to a group of workers within the distributed deep learning framework that are connected to a given worker.) Computing simple average of model parameters during the communication step effectively leads to treating all the workers over which the average is computed equally, regardless of their learning capabilities. \textbf{\textit{What is new in this paper?}} In this paper, we propose AL-DSGD, a decentralized distributed SGD algorithm with a novel \textit{averaging strategy} that assigns specific weights to different neighbor learners based on both their performance and degree, and applies a \textit{corrective force} dictated by both the currently best-performing and the highest degree neighbor when training to accelerate the convergence and improve the generalization performance.

Furthermore, the convergence rate of decentralized SGD is influenced by the network topology. A dense topology demands more communication time~\cite{wang2022matcha}, despite converging fast iteration-wise~\citep{wang2022matcha,koloskova2020unified}, while a sparse topology or one with imbalanced degrees across learners (i.e., nodes have degrees that vary a lot; we will refer to this topology as imbalanced topology) results in a slower convergence in terms of iterations but incurs less communication delays~\citep{wang2019matcha}. Previously proposed solutions addressing the problem of accelerating convergence without sacrificing the performance of the system include bias-correction techniques~\citep{yuan2020influence, huang2022improving, yuan2021removing, tang2018d, he2022rcd, yuan2021removing, alghunaim2022unified}, periodic global averaging or multiple partial averaging methods for reducing frequent communication~\citep{chen2021accelerating, benjamini2014mixing, wang2019slowmo, berahas2018balancing, kong2021consensus}, methods that design new topologies~\citep{nedic2018network, chow2016expander, koloskova2020unified}, or techniques utilizing the idea of communicating more frequently over connectivity-critical links and at the same time using other links less frequently as is done in the SOTA algorithm called MATCHA~\citep{wang2022matcha,wang2021cooperative}. However, these proposed solutions rely on simple model averaging or gradient summation during the communication step. Moreover, many of these solutions use an optimizer that is dedicated to single GPU training and naively adapt it to distributed optimization. It still remains a challenge to design a strategy for handling the communication topology that at the same time allows for fast convergence in terms of iterations and carries low communication burden resulting in time-wise inexpensive iterations. \textbf{\textit{What else is new in this paper?}} Our proposed AL-DSGD algorithm addresses this challenge by employing \textit{dynamic communication graphs}. The dynamic graphs are a composite structure formed from a series of communication graphs. The communication graph and weight matrices follow established practices\citep{wang2019matcha, lian2017can}. Each matrix W(k) is symmetric and doubly stochastic, ensuring nodes converge to the same stationary point\citep{wang2019matcha}. Our method switches between different communication graphs during training thus allowing workers to communicate with more neighbors than when keeping the topology fixed without increasing the total degree of the graph\footnote{Total degree is the total number of links in the communication network topology.}. That equips our method with robustness to the problems faced by imbalanced and sparse topologies, and allows to achieve fast convergence and improved generalization performance in communication-constrained environments.

\textbf{\textit{What else is our contribution in this paper?}} The empirical analysis demonstrating that AL-DSGD converges faster, generalizes better, and is more stable in the communication-constrained environments compared to the SOTA approaches, and the theoretical analysis showing that the AL-DSGD algorithm that relies on dynamic communication graphs indeed achieves a sublinear convergence rate is all new. The proposed AL-DSGD is a meta-scheme algorithm that can be applied to most decentralized SGD algorithms. Given any decentralized algorithm as a baseline, AL-DSGD can accelerate its convergence while also improve robustness to imbalanced and sparse topologies. Finally, we release a general and concise PyTorch-based library for distributed training of deep learning models that supports easy implementation of any distributed deep learning approach ((a)synchronous, (de)centralized). The library is attached to the Supplementary materials and will be released publicly. 

The paper is organized as follows: Section~\ref{sec:PM} contains the preliminaries, Section~\ref{sec:Motivations} motivates the proposed AL-DSGD algorithm, Section~\ref{sec:Method} presents the algorithm, Section~\ref{sec:Theory} captures the theoretical analysis, Section~\ref{sec:Experiments} reports the experimental results, and finally Section~\ref{sec:Conclusions} concludes the paper. Supplementary materials contain proofs, additional derivations, in-depth description of the experiments, as well as additional empirical results.


\section{Preliminaries}
\label{sec:PM}

\subsection{Distributed Optimization Framework}
Distributed machine learning models are trained on data center clusters, where workers can communicate with each other, subject to communication time and delays. Consider a distributed SGD system with $m$ worker nodes. The model parameters are denoted by $x$ where $x \in R^d$. Each worker node $i$ sees data from its own local data distribution $D_i$. The purpose of distributed SGD is to train a model by minimizing the objective function $F(x)$ using $m$ workers. The problem can be defined as follows: 
\begin{align}
\min_{x\in R^d} F(x) &= \min_{x\in R^d} \frac{1}{m} \sum_{i = 1}^m F_i(x)
\notag\\
&=\min_{x\in R^d} \frac{1}{m} \sum_{i = 1}^m E_{s\sim D_i}[l(x;s)],
\end{align}

where $l(x)$ is the loss function and $F_i(x)=E_{s\sim D_i}[l(x;s)]$ is the local objective function optimized by the $i$-th worker.
\subsection{Decentralized SGD (D-PSGD)}

Decentralized distributed SGD can overcome the communication bottleneck problem of centralized SGD~\cite{lian2017can,scaman2018optimal,wang2022matcha, he2022accelerating}. In D-PSGD (also referred to as consensus-based distributed SGD), workers perform one local update and average their models only with neighboring workers. The update rule is given as:
\begin{equation}
x_{k+1,i} = \sum_{j=1}^m W_{ij}\left[x_{k,j} - \eta g_j(x_{k,j};\xi_{k,j})\right],
\end{equation}

where $x_{k,j}$ denotes the model parameters of worker $j$ at iteration $k$, $\xi_{k,j}$ denotes a batch sampled from a local data distribution of worker $j$ at iteration $k$, $W \in \mathbb{R}^{m\times m}$ and $W_{ij}$ is the $(i,j)$-th element of the mixing matrix $W$ which presents the adjacency of node $i$ and $j$. $W_{ij}$ is non-zero if and only if node $i$ and node $j$ are connected. Consequently, sparse topologies would correspond to sparse mixing matrix $W$ and dense topologies correspond to matrix $W$ with most of the entries being away from $0$. 

\section{Motivations}
\label{sec:Motivations}


\begin{figure*}[t!]
    \centering
    \includegraphics[width=130mm]{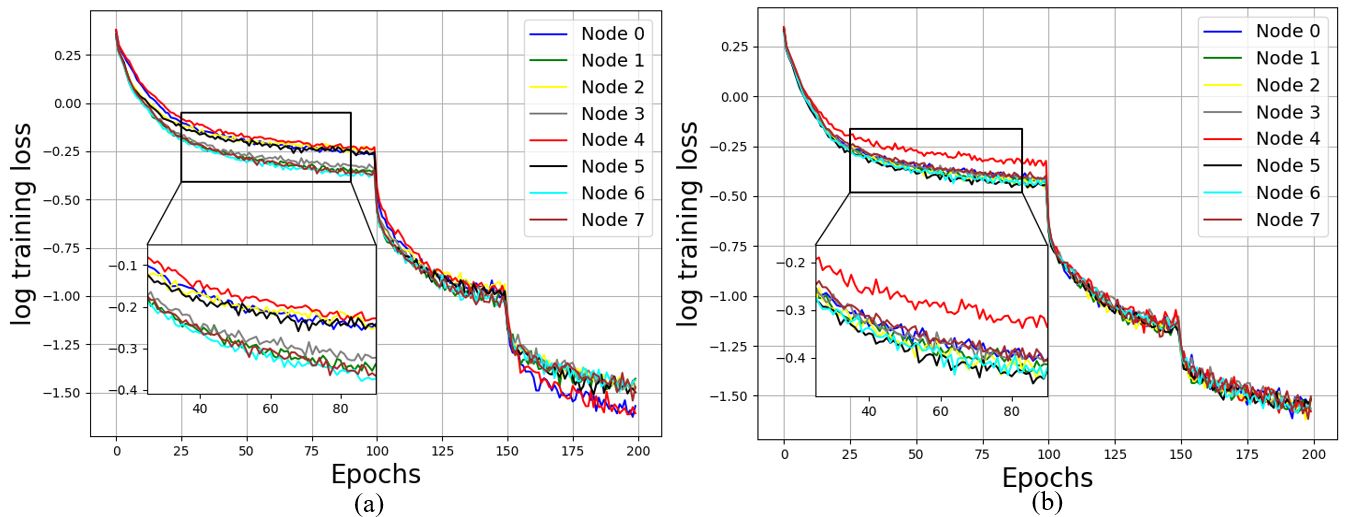}
    \centering
    \caption{Workers with lower degree have worse performance. (a) is the performance of D-PSGD and (b) is the performance of MATCHA. Results were obtained on CIFAR-10 data set using ResNet-50.}
    \label{Fig: 2}
\end{figure*}


In this section we motivate our AL-DSGD algorithm. We start from discussing two SOTA decentralized SGD approaches, D-PSGD and MATCHA, and next show that both D-PSGD and MATCHA suffer from, what we call, the lower degree-worse performance phenomenon. 
\begin{figure}[H]
    \centering
    \includegraphics[width=40mm]{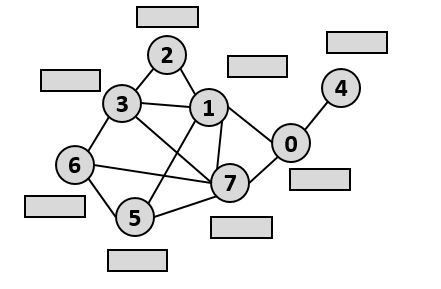}
    \caption{Illustration of decentralized SGD algorithm.}
    \label{Fig: 1}
\end{figure}

\subsection{D-PSGD and MATCHA: Overview}
An overview of the base communication topology for existing decentralized SGD methods can be found in Figure \ref{Fig: 1}. Here, we first compare the construction of the sequence of weight matrices (or in other words mixing matrices) $\{W^{(k)}\}$ ($k$ is the iteration counter) for D-PSGD and MATCHA. In the case of the D-PSGD algorithm, the weight matrices remain the same (denoted as $W$) for the whole training process and thus the communication topology remains fixed the whole time.
The communication between workers during training occurs over a computational graph, where the workers are represented as nodes and the edges denote the connections between them. The connectivity pattern is captured in $W$. Similarly to D-PSGD, MATCHA starts with a pre-defined communication network topology. In contrast to D-PSGD, MATCHA allows the system designer to set a flexible communication budget $c_b$, which represents the average frequency of communication over the links in the network. The sequence of the weight matrices $\{W^{(k)}\}$ is constructed dynamically and depends on the choice of the budget $c_b$.
When $c_b$ = 1, MATCHA is equivalent to vanilla D-PSGD algorithm. When $c_b < 1$, MATCHA carefully reduces the frequency of communication over each link (how quickly this is done depends upon the importance of the link in maintaining the overall connectivity of the graph). In addition, MATCHA assigns probabilities to connections between workers, thus they may become active in some iterations. 
\subsection{Lower Degree - Worse Performance Phenomenon}

\begin{table}[htb]
\centering
\resizebox{\columnwidth}{!}{%
\begin{tabular}{lllllllll}
\hline
\multicolumn{1}{c}{\textbf{}}       &\multicolumn{3}{c}{\textbf{D-PSGD}}                                                                                            &\multicolumn{3}{c}{\textbf{MATCHA}}      \\ \hline
\multicolumn{1}{c}{\textbf{Node}}   &\multicolumn{1}{c}{\textbf{Loss$^1$}} & \multicolumn{1}{c}{\textbf{Loss$^2$}}   &\multicolumn{1}{c}{\textbf{TEST ACC}}  &\multicolumn{1}{c}{\textbf{Loss$^1$}} &\multicolumn{1}{c}{\textbf{Loss$^2$}}   &\multicolumn{1}{c}{\textbf{TEST ACC}}\\ \hline
\multicolumn{1}{c}{0}               &\multicolumn{1}{c}{0.54} & \multicolumn{1}{c}{\textbf{0.027}$\downarrow$}   &\multicolumn{1}{c}{\textbf{87.95}$\downarrow$} &\multicolumn{1}{c}{0.39} & \multicolumn{1}{c}{0.030}   &\multicolumn{1}{c}{93.78}\\ 
\multicolumn{1}{c}{1}               &\multicolumn{1}{c}{0.45} & \multicolumn{1}{c}{0.037}   &\multicolumn{1}{c}{92.11} &\multicolumn{1}{c}{0.36} & \multicolumn{1}{c}{0.030}   &\multicolumn{1}{c}{93.68}\\ 
\multicolumn{1}{c}{2}               &\multicolumn{1}{c}{0.58} & \multicolumn{1}{c}{0.035}   &\multicolumn{1}{c}{92.21} &\multicolumn{1}{c}{0.36} & \multicolumn{1}{c}{0.029}   &\multicolumn{1}{c}{93.76}\\ 
\multicolumn{1}{c}{3}               &\multicolumn{1}{c}{0.45} & \multicolumn{1}{c}{0.034}   &\multicolumn{1}{c}{92.36} &\multicolumn{1}{c}{0.38} & \multicolumn{1}{c}{0.027}   &\multicolumn{1}{c}{93.91}\\ 
\multicolumn{1}{c}{4}               &\multicolumn{1}{c}{\textbf{0.59}} & \multicolumn{1}{c}{\textbf{0.025}$\downarrow$}   &\multicolumn{1}{c}{\textbf{87.86}$\downarrow$} &\multicolumn{1}{c}{\textbf{0.47}} & \multicolumn{1}{c}{0.026}   &\multicolumn{1}{c}{93.68}\\ 
\multicolumn{1}{c}{5}               &\multicolumn{1}{c}{0.55} & \multicolumn{1}{c}{0.032}   &\multicolumn{1}{c}{92.25} &\multicolumn{1}{c}{0.36} & \multicolumn{1}{c}{0.031}   &\multicolumn{1}{c}{93.81}\\ 
\multicolumn{1}{c}{6}               &\multicolumn{1}{c}{0.42} & \multicolumn{1}{c}{0.037}   &\multicolumn{1}{c}{92.38} &\multicolumn{1}{c}{0.37} & \multicolumn{1}{c}{0.027}   &\multicolumn{1}{c}{93.72}                 \\ 
\multicolumn{1}{c}{7}               &\multicolumn{1}{c}{0.44} & \multicolumn{1}{c}{0.034}   &\multicolumn{1}{c}{92.32} &\multicolumn{1}{c}{0.38} & \multicolumn{1}{c}{0.031}   &\multicolumn{1}{c}{93.82}                 \\ \hline
\end{tabular}    
}
\caption{\label{tab:widgets} Performance of D-PSGD and MATCHA on CIFAR-10. Loss$^1$ is the training loss computed  at the $100^{\text{th}}$ epoch on a local data set seen by the node. Loss$^2$ is the training loss computed at the $200^{\text{th}}$ epoch on a local data set seen by the node.}
\label{Tab:1}
\end{table}

Previous literature~\citep{wang2019matcha} explored how topology affects performance in distributed decentralized SGD optimizers. Denser networks lead to quicker error convergence in terms of iterations but introduce longer communication delays. In addition we discuss the lower degree-worse performance phenomenon in this section. In particular we show that the worker with lower degree will converge slower than the other nodes, and achieve worse final model at the end of training. We use Figure~\ref{Fig: 2} and Table~\ref{Tab:1} to illustrate this phenomenon (refer to Section~\ref{sec:Experiments} regarding experimental details). Figure~\ref{Fig: 2} shows that the node denoted as Node 4 achieves the slowest convergence before the first drop of the learning rate from among all of the workers. 

As seen in Figure~\ref{Fig: 1}, Node 4 possesses the lowest degree, with only Node 0 connected to it. Due to the weight matrix $W$ design, weaker node performance adversely affects neighbors. Notably, in D-PSGD, Node 0 and Node 4 display reduced local training loss, yet their test accuracy lags behind other nodes. This is because the training loss is calculated using a local subset seen by the worker, and a lack of communication results in the over-fitting of the local model. This experiment shows that lower degree leads to slower convergence and worse final model (this phenomenon is observed in every single experiment with different random seeds).
In both D-PSGD and MATCHA, the output model is the average of models from all of the workers. Workers with lower degrees will naturally deteriorate the performance of the entire system. Therefore, a natural question arises: how to improve the performance of workers with low degrees and accelerate the overall convergence of the entire system without increasing the density of the network topology?

\section{Proposed Method}
\label{sec:Method}
\begin{algorithm}[t!]
    \begin{algorithmic}[1]
    \caption{Proposed AL-DSGD algorithm}
    \label{Alg: 2}
    \STATE \textbf{Initialization:}initialize local models $\{x_0^i\}_{i=1}^m$ with the \textbf{different} initialization, learning rate $\gamma$, weight matrices sequence $\{W^{(k)}\}$, and the total number of iterations $K$. Initialize communication graphs set $\{G_{(i)}\}_{i=1}^n$, each communication graph $\{G_{(i)}\}$has its own weight matrices sequence $\{W^{(k)}\}$. Pulling coefficients $\lambda_N$ and $\lambda_{\tau}$. Model weights coefficients $w_N$ and $w_\tau$. Split the original dataset into subsets.\;
    \WHILE{k=0, 1, 2, ... K-1 $\leq$ K}
            \STATE Compute  the local stochastic gradient $\nabla F_i(x_{k, i}; \xi_{k, i})$ on all workers\;
            \STATE For each worker, fetch neighboring models and determine the adjacent best worker $x_{k,i}^N$ and the adjacent maximum degree worker: $x_{k,i}^\tau$.
            \STATE Update the local model with corrective force: 
            \begin{align*}
                x_{k+\frac{1}{2}, i} &= x_{k, i} - \gamma \nabla F_i(x_{k, i}; \xi_{k, i}) - \gamma \lambda_N (x_{k, i} - x_{k,i}^N) \\
                &- \gamma \lambda_{\tau} (x_{k, i} - x_{k,i}^\tau)
            \end{align*}\;
            \STATE Average the model with neighbors, give additional weights to worker $x_{k,i}^N$ and $x_{k,i}^\tau$: 
            \begin{align*}
                x_{k+1, i} &= (1 - w_N - w_{\tau}) \cdot \left(\right. \sum_{j=1,j\ne i}^m W_{ij}^{(k)}x_{k,j}\\
                \vspace{-0.5in}
                &+ W_{ii}^{(k)} x_{k+\frac{1}{2},i}\left.\right)+ w_N \cdot x_{k,i}^N + w_{\tau} \cdot x_{k,i}^\tau
            \end{align*}\;
            \STATE Switch to new communication graph.\; 
    \ENDWHILE
    \STATE \textbf{Output:}the average of all workers $\frac{1}{m} \sum_{i=1}^m x_{k,i}$\;
    \vspace{-0.03in}
    \end{algorithmic}
\end{algorithm}

\begin{figure*}[t!]
    \centering
    \includegraphics[width=160mm]{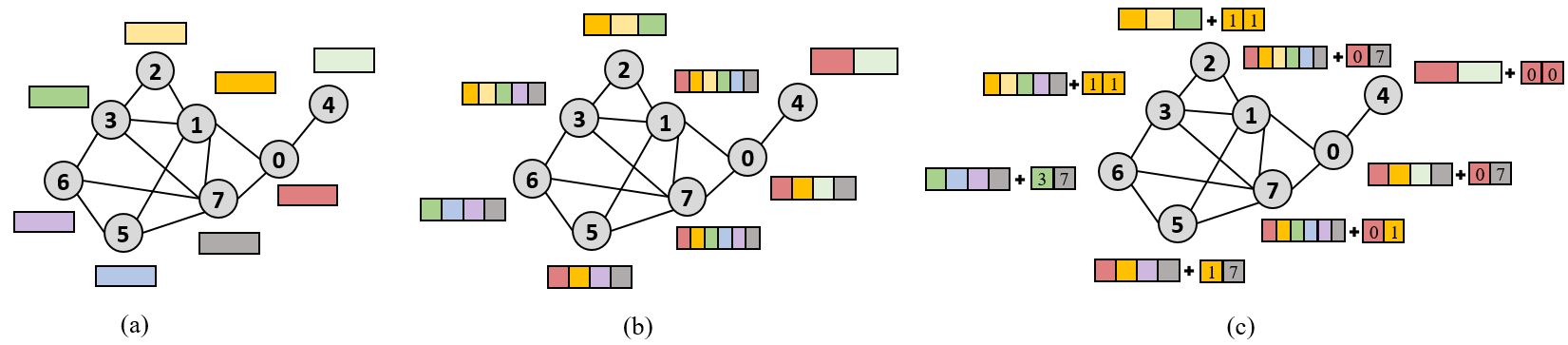}
    \centering
    \caption{(a) The weights before communication are represented as colored blocks, where different colors correspond to different workers. (b) Previous methods simply average the training model with neighbors. Each colored block denotes the identity of workers whose parameters were taken to compute the average. (c) To illustrate AL-DSGD, we assume that the higher is the index of the worker, the worse is its performance in this iteration. For each node, in addition to averaging with neighboring models, AL-DSGD assigns additional weights to the best performing adjacent model and the maximum degree adjacent model. This is depicted as the sum, where the additional block has two pieces (the left corresponds to the best performing adjacent model and the right corresponds to the maximum degree adjacent model; the indexes of these models are also provided). For example, in the case of model $2$, both the best-performing adjacent model and the maximum degree adjacent model is model $1$.}
    \label{Fig: 3}
\end{figure*}

We present the Adjacent Leader Decentralized Gradient Descent(AL-DSGD) in Algorithm ~\ref{Alg: 2} and discuss it in details below. A \textit{\textbf{motivation example}} for Algorithm ~\ref{Alg: 2} can be found in figure~\ref{Fig: 4}. A \textit{\textbf{visualization}} of step 6 in Algorithm ~\ref{Alg: 2} can be found in Figure~\ref{Fig: 3} and a \textit{\textbf{visualization}} of step 7 (dynamic communication graphs) in Algorithm ~\ref{Alg: 2} can be found in Figure~\ref{Fig: 11}. Let $k$ denote the iteration index ($k = 1,2,\dots,K$) and $i$ denote the index of the worker ($i = 1,2,\dots,m$). Let $x_{k,i}^N$ denote the best performing worker from among the workers adjacent to node $i$ at iteration $k$ and let $x_{k,i}^\tau$ denote the maximum degree worker from among the workers adjacent to node $i$ at iteration $k$. Let $\{G_{(i)}\}_{i=1}^n$ denote the dynamic communication graphs, each communication graph $\{G_{(i)}\}$has its own weight matrices sequence $\{W^{(k)}\}$. In the upcoming paragraphs of this section, we present novel contributions corresponding to Algorithm~\ref{Alg: 2}.

\textit{\textbf{The Corrective Force:}} To address the problem of detrimental influence of low degree nodes on the entire learning system, we first increase the influence of the workers with better performance ($x_{k,i}^N$) and larger degrees ($x_{k,i}^\tau$). Specifically, in the communication step, workers send their model parameters, local training loss, and local degree to their neighbors. At the end of the communication, we determine the adjacent best worker ($x_{k,i}^N$) based on training loss and the adjacent maximum degree worker ($x_{k,i}^\tau$) based on local degree for each node. Then, at training we introduce an additional corrective "force" pushing workers to their adjacent nodes with the largest degrees and the lowest train loss according to the following:

\vspace{-0.2in}
\begin{equation}
\begin{split}
x_{k+\frac{1}{2}, i}\! &=\! x_{k, i}\!- \!\gamma \nabla F_i(x_{k, i}; \xi_{k, i}) \\
&- \gamma \lambda_N (x_{k, i} - x_{k,i}^N) - \gamma \lambda_{\tau} (x_{k, i} - x_{k,i}^\tau) \nonumber
\end{split}
\end{equation}

where $\lambda_N$ and $\lambda_{\tau}$ are pulling coefficients, $\gamma$ is the learning rate, and $\nabla F_i(x_{k, i}; \xi_{k, i})$ is the gradient of the loss function for worker $i$ computed on parameters $x_{k,i}$ and local data sample $\xi_{k,i}$ that is seen by worker $i$ at iteration $k$. This update is done in step 5 of the Algorithm~\ref{Alg: 2}. Finally, note that \textit{no additional computations or communication steps are required} to acquire the information about the best performing adjacent worker or the maximum degree adjacent worker for a given node. This is because during the training process we compute the training losses, and furthermore each worker must be aware of its degree and the degree of the adjacent worker before communication. Figure~\ref{Fig: 4} intuitively illustrates that adding the corrective force can accelerate the convergence rate of the worker with low degree and worse performance. 

\textbf{\textit{The Averaging Step:}} Secondly, when averaging workers, we weight them according to their degree and performance (see step 6). This is done according to the formula:

\vspace{-0.3in}
\begin{equation}
\begin{split}
x_{k + 1, i} \!= &(1 - w_N - w_{\tau}) \cdot (W_{ii}^{(k)} x_{k+\frac{1}{2},i} \!+\!\!\sum_{j=1,j\ne i}^m W_{ij}^{(k)}x_{k,j})\\
&+ w_N \cdot x_{k,i}^N + w_{\tau} \cdot x_{k,i}^\tau \nonumber
\end{split}
\end{equation}

We visualize the process in Figure~\ref{Fig: 3}. Take node $0$ as an example. Figure \ref{Fig: 3}(a) to \ref{Fig: 3}(b) shows the classical communication step in vanilla decentralized SGD algorithms. Since the node $7$, that is adjacent to node $0$, has the largest degree and node $0$ itself is the best-performance worker, we increase the weight of node $7$ and node $0$ in Figure \ref{Fig: 3}(c).

\textbf{\textit{The Dynamic Communication Graph:}} Third, instead of relying on a single communication graph, we introduce $n$ graphs with different topologies and switch between them (see Figure~\ref{Fig: 11} for the special case of $n=3$). The total degree of each graph is the same as for the baseline. When we switch the graph, we are indeed altering the physical network topology. Since distributed machine learning models are trained on data center clusters, where workers can communicate with each other, \textit{switching communication graph won't lead to additional training time.}
We randomly choose a graph to start with and switch between different graphs after each training iteration. 
By using the dynamic communication graph, the workers are connected to more neighbors. This is done without increasing the total degree of the graph. This allows us to balance the expected degree of each node and avoid the poor performance of nodes with extremely low degrees.

Finally, we would like to emphasize that ALD-SGD is a meta-scheme that can be put on the top of any decentralized SGD method, including D-PSGD and MATCHA.
\vspace{-0.15in}
\begin{figure}[t]
    \centering
    \includegraphics[width=30mm]{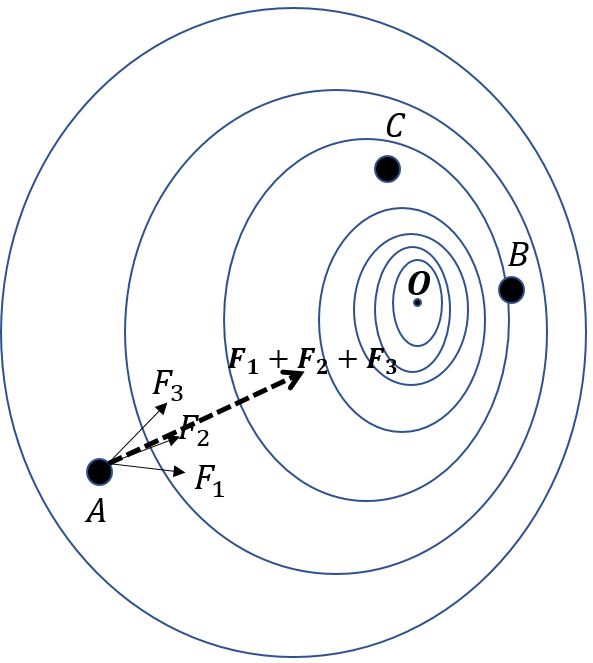}
    \centering
    \vspace{-0.1in}
    \caption{
    Motivating example: In Algorithm~\ref{Alg: 2} step 5, Point A represents a worker model with low degree and poor performance. $F_1$ is the data batch gradient, $F_2$ is the corrective force from the best performing adjacent worker, and $F_3$ is the corrective force from the adjacent worker with the highest degree. Point B represents the best performing adjacent node to A, while Point C represents the adjacent node with the maximum degree. Point O represents the optimum. Note that $F_1 + F_2 + F_3$ directs to the optimum, highlights the benefit of corrective force in optimization.}
    \label{Fig: 4}
\end{figure}
\begin{figure}[H]
    \centering
    \includegraphics[width=80mm]{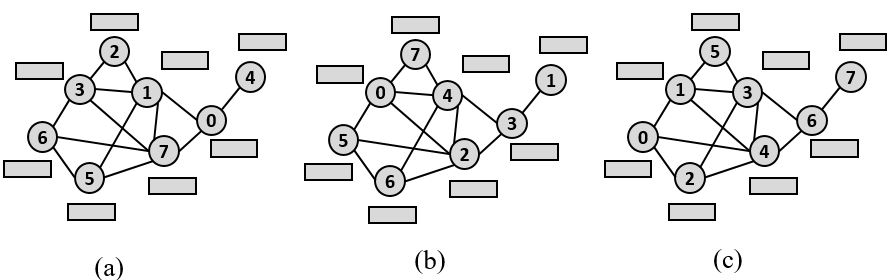}
    \centering
    \caption{AL-DSGD with three Laplacian matrices rotates workers locations between (a), (b), and (c).}
    \label{Fig: 11}
\end{figure}

\section{Theoretical analysis}
\label{sec:Theory}

This section offers theoretical analysis for the proposed AL-DSGD algorithm. As a meta-scheme adaptable to various decentralized SGD methods, our analysis focuses on embedding AL-DSGD with the MATCHA core (where D-PSGD is a MATCHA special case, making our analysis broadly applicable). The structure is as follows: Section \ref{sec:5_1} establishes the update formula for AL-DSGD atop MATCHA, incorporating a dynamic communication graph. Notably, this subsection includes a convergence theorem. In Section \ref{sec:5_2}, we demonstrate that, subject to specified assumptions, a range of hyperparameters $\alpha$, $\omega_N$, $\omega_\tau$, pull coefficients $\lambda_N$ and $\lambda_\tau$, along with the learning rate $\eta$, exist, resulting in AL-DSGD achieving a sublinear convergence rate. Thus our theoretical guarantee matches that of MATCHA and D-PSGD and shows that by using dynamic graphs we do not loose in terms of convergence compared to these schemes.

\vspace{-0.1in}
\subsection{Averaged weight matrix}\label{sec:5_1}

From Algorithm\ref{Alg: 2}, the model average step for our AL-DSGD algorithm is:
\vspace{-0.25in}
\begin{align}\label{eq:update}
    x_{k+1,i}=&\overbrace{(1-\omega_N-\omega_\tau)\sum_{j=1}^mW_{ij}^{(k)}x_{k,j}+\omega_Nx_{k,i}^n+\omega_\tau x_{k,i}^\tau}^{(I)}\notag\\
    &-\gamma(1-\omega_N-\omega_\tau)W_{ii}^{(k)}\left[\right.\nabla F_i(x_{k, i}; \xi_{k, i}) \nonumber\\
    &+ \lambda_N (x_{k, i} - x_{k,i}^N) + \lambda_{\tau} (x_{k, i} - x_{k,i}^\tau)\left.\right]
\end{align}
In this section, we only consider part (I) in formula (\ref{eq:update}) without the gradient update step. We define $\widetilde{x}_{k+\frac{1}{2},i}:=(I)$, and 
denote $X_k=[x_{k,1},x_{k,2},...,x_{k,m}]$, $\widetilde{X}_{k+\frac{1}{2}}=[\widetilde{x}_{k+\frac{1}{2},1},\widetilde{x}_{k+\frac{1}{2},2},...,\widetilde{x}_{k+\frac{1}{2},m}]$, $X_k^N\!=\![x_{k,1}^N,x_{k,2}^N,...,x_{k,m}^N]$ and $X_k^\tau=[x_{k,1}^\tau,x_{k,2}^\tau,...,x_{k,m}^\tau]$, we have
\begin{align}\label{eq:x2}
    \widetilde{X}_{k+\frac{1}{2}}&=X_k\widetilde{W}^{(k)}\notag\\
    \widetilde{W}^{(k)}&=(1-\omega_N-\omega_\tau)W^{(k)}+\omega_NA_k^N+\omega_\tau A_k^\tau\notag\\
    W^{(k)}&=1-\alpha L^{(k)}.
\end{align}
\\
where $L^{(k)}$ denotes the graph Laplacian matrix at the $k^\text{th}$ iteration, $X_k^N$ and $X_k^\tau$ are the model parameter matrix of the adjacent best workers and adjacent maximum degree workers at the $k^\text{th}$ iteration. Assume $X_k^N=X_kA_k^N$, $X_k^\tau=X_kA_k^\tau$. Since every row in $X_k^N$ and $X_k^\tau$ is also a row of the model parameter matrix $X_k$, we could conclude that the transformation matrices $A_k^N$ and $A_k^\tau$ must be the left stochastic matrices.



AL-DSGD switches between $n$ communication graphs $\{G_{(i)}\}_{i=1}^n$. Let $\{L_{(i),j}\}_{j=1}^m$ be the Laplacian matrices set as matching decomposition of graph $G_{(i)}$. Led by MATCHA approach, to each matching of $L_{(i),j}$ to graph $G_{(i)}$ we assign an independent Bernoulli random variable $B_{(i),j}$ with probability $p_{(i),j}$ based on the communication budget $c_b$. 
Then the graph Laplacian matrix at the $k^\text{th}$ iteration $L^{(k)}$ can be written as:~$\sum_{j=1}^mB^{(k)}_{(1),j}L_{(1),j} \text{(if $k$ mod $n = 1$)}$, $\sum_{j=1}^mB^{(k)}_{(2),j}L_{(2),j} \text{(if $k$ mod $n = 2)$}$,\ldots, $\sum_{j=1}^mB^{(k)}_{(n),j}L_{(n),j} ~\text{(if $k$ mod $n = 0)$}$.
The next theorem captures the convergence of the AL-DSGD algorithm.

\begin{theorem}\label{thm:1}
    Let $\{L^{(k)}\}$ denote the sequence of Laplacian matrix generated by AL-DSGD algorithm with arbitrary communication budget $c_b>0$ for the dynamic communication graph set $\{G_{(i)}\}_{i=1}^n$. Let the mixing matrix $\widetilde{W}^{(k)}$ be defined as in Equation~\ref{eq:x2}). There exists a range of $\alpha$ and a range of average parameters $\omega_N=\omega_\tau\in(0,\omega(\alpha))$, whose bound is dictated by $\alpha$, such that the spectral norm $\rho=\max\left\{\norm{\mathbb{E}\left[\widetilde{W}^{(k)}(I\!-\!J)\widetilde{W}^{(k)\intercal}\right]}\!, \norm{\mathbb{E}\left[\widetilde{W}^{(k)}\widetilde{W}^{(k)\intercal}\right]}\right\}\!<\!1$, where $J=\mathbf{1}\mathbf{1}^\intercal/m$.
\end{theorem}

Theorem \ref{thm:1} states that for arbitrary communication budget $c_b$ there exists some $\alpha$, $\omega_N$ and $\omega_\tau$ such that the spectral norm $\rho<1$, which is a necessary condition for AL-DSGD to converge.

\subsection{Convergence guarantee}\label{sec:5_2}
This section provides the convergence guarantee for the proposed AL-DSGD algorithm. We define the average iterate as $\overline{x}^{(k)}=\frac{1}{m}\sum_{i=1}^mx_i^{(k)}$ and the minimum of the loss function as $F^*$. This section demonstrates that the averaged gradient norm $\frac{1}{K}\sum_{k=1}^K\mathbb{E}[\norm{\nabla F(\overline{x}^{(k)})}]$ converges to zero with sublinear convergence rate.

\begin{assumption}\label{asm:1}
We assume that the loss function $F(x)=\sum_{i=1}^mF_i(x)$ satisfy the following conditions:
\begin{itemize}
    \item[(1)]\textit{Lipschitz continuous:} $\norm{F_i(x)-F_i(y)}\leq\beta\norm{x-y}$\\
    \item[(2)]\textit{Lipschitz gradient:} $\norm{\nabla F_i(x)\!-\!\nabla F_i(y)}\!\leq\!L\norm{x\!-\!y}$\\
    \item[(3)]\textit{Unbiased gradient:} 
    $\mathbb{E}_{\xi_i}[g_i(x_k;\xi_i)]=\nabla F_i(x)$\\
    \item[(4)]\textit{Bounded variance:} $\mathbb{E}_{\xi_i}[\norm{g_i(x_k;\xi_i)-\nabla F_i(x)}^2]\!\leq\!\sigma^2$\\
    \item[(5)]\textit{Unified gradient:} $\mathbb{E}_{\xi_i}[\norm{\nabla F_i(x)-\nabla F_i(x)}^2]\leq\zeta^2$\\
    \item[(6)]\textit{Bounded domain:}\!\! $\max\{\|x_{k,i}\!\!-\!x_{k,i}^N\|,\!\|x_{k,i}\!\!-\!x_{k,i}^\tau\|\}\!\leq\!\Delta^2$.
\end{itemize}
\end{assumption}

\begin{theorem}\label{thm:2}
    Suppose all local workers are initialized with $\overline{x}^{(1)}=0$ and $\{\widetilde{W}^{(k)}\}_{k=1}^K$ is an i.i.d matrix sequence generated by AL-DSGD algorithm which satisfies the spectral norm $\rho<1$ ($\rho$ is defined in Section \ref{sec:5_1}). Under Assumption \ref{asm:1}, if $\lambda=2\lambda_N=2\lambda_\tau$ and $(1-\alpha)(1-\omega)\gamma L\leq\min\{1,(\sqrt{\rho^{-1}}-1)\}$, then after K iterations:
    \begin{align*}
     \frac{1}{K}\sum_{i=1}^K\mathbb{E}\left[\norm{\nabla F(\overline{x}_k)}^2\right] \leq\frac{8(F(\overline{x}_{1})-F^*)}{\eta K}+\frac{8M}{\eta} \\
     +\frac{8\eta^2L^2\rho}{1-\sqrt{\rho}}\left(\frac{m\sigma^2+\lambda^2\Delta^2}{m(1+\sqrt{\rho})}+\frac{3\zeta^2}{1-\sqrt{\rho}}\right),
    \end{align*}
    
    where $\eta=(1-\alpha)(1-\omega)\gamma$ and $M=\frac{\eta^2L\sigma^2}{2m}+\lambda\eta\beta\Delta+\lambda\eta^2L\beta \Delta+\frac{\lambda^2\eta^2L\Delta^2}{2}$. When setting $\lambda=\sqrt{\frac{m}{K}}$, $\gamma=\sqrt{\frac{m}{(1-\omega)(1-\alpha)K}}$, we obtain sublinear convergence rate.
\end{theorem}
Note that all assumptions in Assumption \ref{asm:1} are commonly used assumptions for decentralized distributed optimization \cite{lian2017can,wang2019matcha,li2017convergence}.
    $(1-\alpha)(1-\omega) L\leq\min\{1,(\sqrt{\rho^{-1}}-1)\}$ is a weak assumption on the learning rate and resembles similar assumption in Theorem 2 in \cite{wang2019matcha}. Note that we give an exact value for the upper-bound on $\rho$ in Appendix \ref{app:thm_1}, which implies that under certain choices of $\alpha$, $\omega_N$, and $\omega_\tau$, $\rho$ could be much smaller than 1 and the right-hand side of the bound is therefore not approaching 0. Moreover, Assumption \ref{asm:1} (1) guarantees the Lipschitz constant for the loss objective function, and constructing learning rate based on the Lipschtz constant is widely used in many convergence proofs \cite{wu2018wngrad,mhammedi2019lipschitz,yedida2021lipschitzlr}.

\begin{figure*}[t!]
    \centering
    \includegraphics[width=150mm]{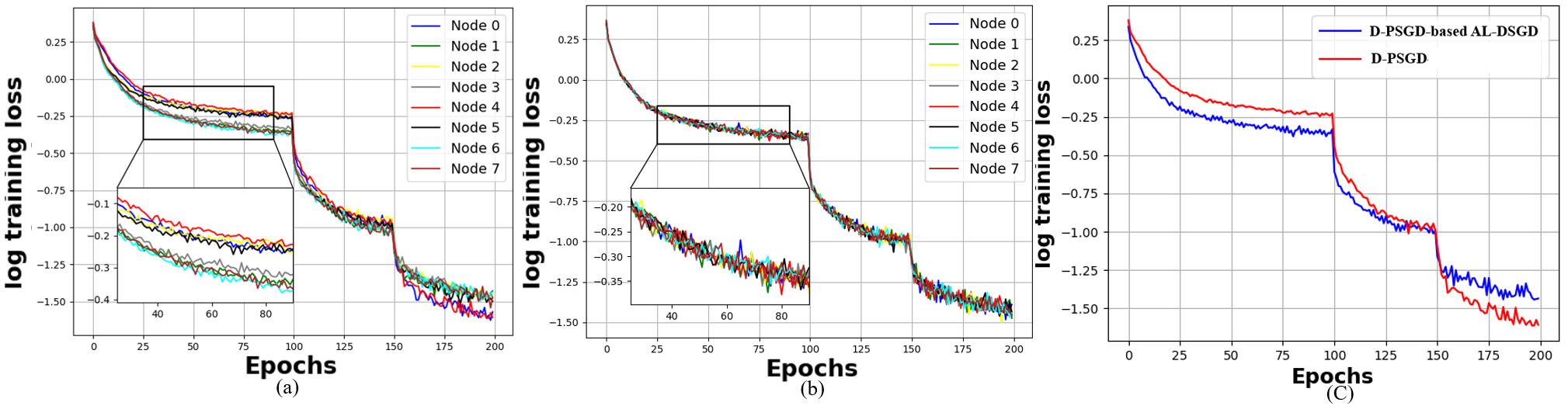}
    \centering
    \caption{Training loss behavior for ResNet-50 trained on CIFAR-10. Optimization schemes: (a) D-PSGD, (b) D-PSGD-based AL-DSGD (c): Comparison between worst performing workers from a and b.}
    \label{Fig: 6}
\end{figure*}

\begin{figure*}[t!]
    \centering
    \includegraphics[width=150mm]{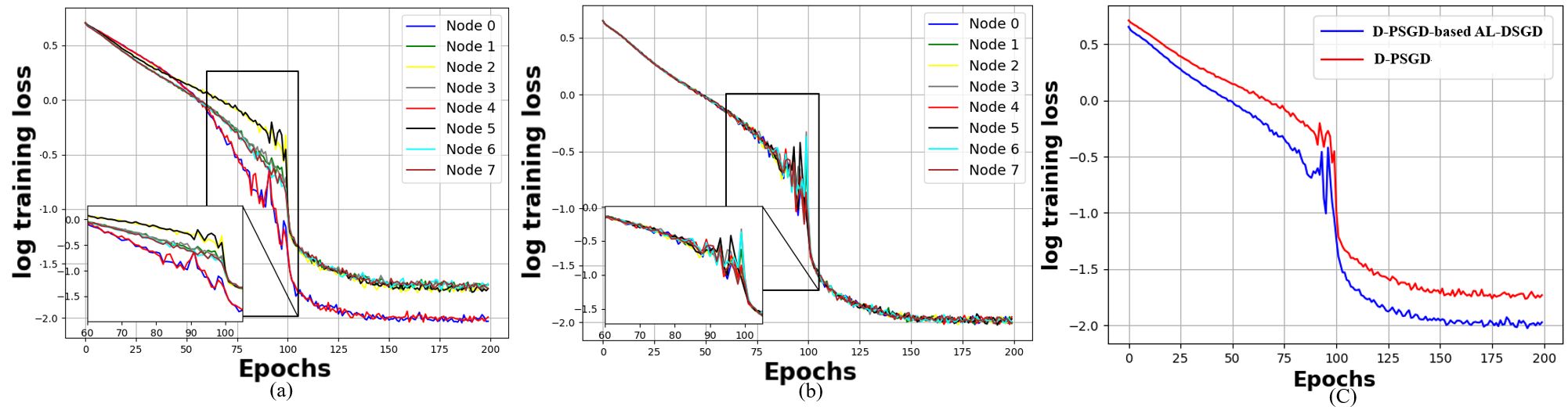}
    \centering
    \caption{Training loss behavior for WideResNet trained on CIFAR-100. Optimization schemes: (a) D-PSGD, (b) D-PSGD-based AL-DSGD  (c): Comparison between worst performing workers from a and b.}
    \label{Fig: 5}
\end{figure*}
\begin{figure*}[t!]
    \centering
    \includegraphics[width=150mm]{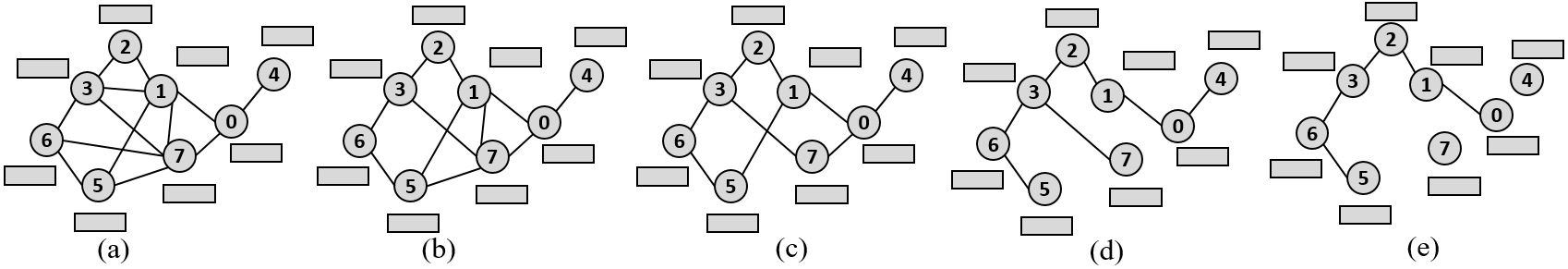}
    \centering
    \caption{(a) Graph with total degree $D = 13, d = 100\%$. (b) Graph with reduced degree $d = 84.6\%, D = 11$, (c) Graph with reduced degree $d = 69.2\%, D = 9$, (d) Graph with reduced degree $d = 53.8\%, D = 7$, (e) Graph with reduced degree $d = 38.5\%, D = 5$.}
    \label{Fig: 12}
\end{figure*}
\section{Experimental results}
\label{sec:Experiments}
This section is devoted to the empirical evaluation of the proposed AL-DSGD scheme.

\textbf{Datasets and models:} 
In our experiments, we employ ResNet-50 and Wide ResNet models. The models are trained on CIFAR-10 and CIFAR-100~\citep{krizhevsky2009learning}. The same architectures and data sets were used by our competitor methods, MATCHA and D-PSGD. The training data set is randomly and evenly partitioned over a network of workers (each worker sees all the classes and the number of samples per classes are the same across all the workers). In the decentralized synchronous setting, a pre-round barrier addresses computational speed variations (straggler issue) caused by hardware and data sampling differences. Slower workers naturally wait for faster ones to complete training before synchronization. This aligns with our baselines~\citep{wang2019matcha, lian2017can}. 


\textbf{Machines/Clusters:} All the implementations are done in PyTorch and OpenMPI within mpi4py. We conduct experiments on a HPC cluster with 100Gbit/s network. In all of our experiments, we use RTX8000 GPU as workers.

\textbf{AL-DSGD and Competitors:} We implemented the proposed AL-DSGD with pulling coefficients $\lambda_N = 0.1$ and $\lambda_{\tau} = 0.1$. We set model weights coefficients to $w_N = 0.1$ and $w_\tau = 0.1$. The pulling coefficients and the model weights are fine-tuned for the ALD-SGD-based D-PSGD with ResNet-50 trained on CIFAR-10 and then used for other experiments.
We compared our algorithm with the D-PSGD and MATCHA methods, where in case of MATCHA the communication budget $c_b$ was set to $c_b = 0.5$, as recommended by the authors.


\textbf{Implementations:} All algorithms are trained for a sufficiently long time until convergence or onset of over-fitting. The learning rate is fine-tuned for the D-PSGD baseline and then used for MATCHA and AL-DSGD algorithm. The initial learning rate is 0.4 and reduced by a factor of 10 after 100 and 150 epochs. The batch size per worker node is 128. 

\begin{figure*}[t!]
    \centering
    \includegraphics[width=120mm]{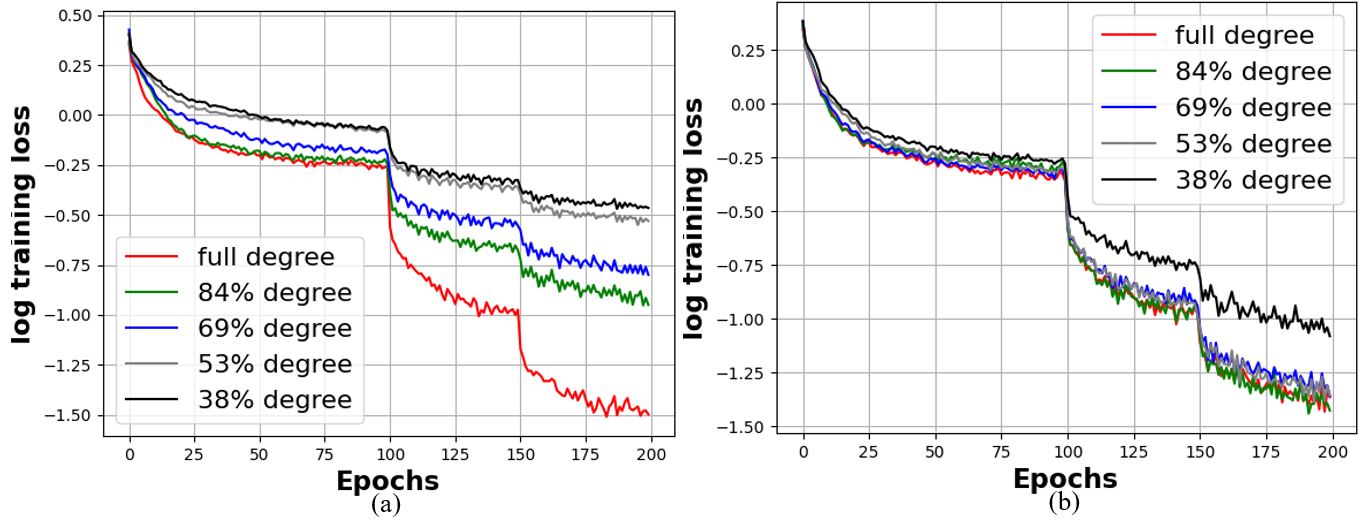}
    \centering
    \caption{ResNet-50 model trained on CIFAR-10. Performance of a) D-PSGD and b) AL-DSGD-based D-PSGD with different topology degrees.}
    \label{Fig: 9}
\end{figure*}

\subsection{Convergence and performance}

\begin{table}[t]
\centering
\resizebox{\columnwidth}{!}{%
\setlength{\tabcolsep}{1.2mm}{
\begin{tabular}{lllllllll}
\hline
\multicolumn{1}{c}{\textbf{}}       &\multicolumn{2}{c|}{\textbf{CIFAR-10/ResNet-50}}      &\multicolumn{2}{c}{\textbf{CIFAR-100/WideResNet}}      \\ 
\hline
\multicolumn{1}{c}{\textbf{Node}}   &\multicolumn{1}{c}{\textbf{D-PSGD}}    &\multicolumn{1}{c|}
{\textbf{AL-DSGD}}    &\multicolumn{1}{c}{\textbf{D-PSGD}}    &\multicolumn{1}{c}{\textbf{AL-DSGD}}    \\ \hline
\multicolumn{1}{c}{0}               &\multicolumn{1}{c}{\textbf{87.95}$\downarrow$} &\multicolumn{1}{c|}{93.68}      &\multicolumn{1}{c}{\textbf{59.45}$\downarrow$} &\multicolumn{1}{c}{76.18}          \\ 
\multicolumn{1}{c}{1}               &\multicolumn{1}{c}{92.11} &\multicolumn{1}{c|}{93.72}   &\multicolumn{1}{c}{74.70} &\multicolumn{1}{c}{76.35}\\ 
\multicolumn{1}{c}{2}               &\multicolumn{1}{c}{92.21} &\multicolumn{1}{c|}{93.55}   &\multicolumn{1}{c}{74.65} &\multicolumn{1}{c}{76.10}\\ 
\multicolumn{1}{c}{3}               &\multicolumn{1}{c}{92.36} &\multicolumn{1}{c|}{93.87}   &\multicolumn{1}{c}{74.16} &\multicolumn{1}{c}{76.36}\\ 
\multicolumn{1}{c}{4}               &\multicolumn{1}{c}{\textbf{87.86}$\downarrow$} &\multicolumn{1}{c|}{93.83}  &\multicolumn{1}{c}{\textbf{59.63}$\downarrow$} &\multicolumn{1}{c}{76.51}\\ 
\multicolumn{1}{c}{5}               &\multicolumn{1}{c}{92.25} &\multicolumn{1}{c|}{93.48}   &\multicolumn{1}{c}{74.62} &\multicolumn{1}{c}{76.34}\\ 
\multicolumn{1}{c}{6}               &\multicolumn{1}{c}{92.38} &\multicolumn{1}{c|}{93.65}   &\multicolumn{1}{c}{74.59} &\multicolumn{1}{c}{76.39}                 \\ 
\multicolumn{1}{c}{7}               &\multicolumn{1}{c}{92.32} &\multicolumn{1}{c|}{93.62}   &\multicolumn{1}{c}{74.57} &\multicolumn{1}{c}{76.30}                 \\ \hline
\multicolumn{1}{c}{\textbf{AVG}}               &\multicolumn{1}{c}{91.18} &\multicolumn{1}{c|}{\textbf{93.68}}   &\multicolumn{1}{c}{70.79} &\multicolumn{1}{c}{\textbf{76.31}}                 \\ \hline
\end{tabular}} 
}

\caption{\label{tab:widgets} Test accuracy obtained with D-PSGD and D-PSGD-based AL-DSGD for ResNet-50 model trained on CIFAR-10 and WideResNet model trained on CIFAR-100.}
\label{Tab:2}
\end{table}

The results from training models with D-PSGD and D-PSGD-based AL-DSGD (AL-DSGD on top of D-PSGD) are in Table~\ref{Tab:2}. MATCHA and MATCHA-based AL-DSGD results are in Appendix~\ref{Appendix: 4}, Table~\ref{Tab:10}. Training loss, shown in Figure~\ref{Fig: 6} and \ref{Fig: 5}, assesses convergence better due to unstable test loss. AL-DSGD reduces variance between nodes and speeds up convergence for the worst-performing node. 

\textbf{\textit{In summary, AL-DSGD has enhanced the test accuracy for both the average and worst-performing models.}} Tables~\ref{Tab:2} and~\ref{Tab:10} reveal AL-DSGD's superior generalization over other methods. For the case of D-PSGD, the test accuracy has increased by resp. $5.8\%$ and $16.7\%$ in the worst-performance worker and by resp. $2.1\%$ and $5.5\%$ in the final averaged model for CIFAR-10/ResNet50 and CIFAR-100/WideResNet tasks, respectively, when putting AL-DSGD on the top of D-PSGD. For the case of MATCHA, even though the AL-DSGD does not dramatically increase the baseline performance for CIFAR-10/ResNet50 task because the model is relatively simple, it strongly outperforms the baseline on more complicated CIFAR100/WideResNet task.  As shown in Figure~\ref{Fig: 13} in the Appendix~\ref{Appendix: 4}, AL-DSGD demonstrates more stable (i.e.,  smaller discrepancies  between nodes) and faster convergence compared to MATCHA. 

We would like to further emphasize that, except for converging to a better optimum, another significant advantage of our AL-DSGD algorithm is that it is much more robust to imbalanced and sparse topology, as will be discussed in the following section.

\subsection{Communication}

In this subsection, we evaluate algorithm performance using ResNet-50 on CIFAR-10 in a communication-constrained environment. \textit{\textbf{The experiments overall demonstrate that AL-DSGD has enhanced test accuracy in scenarios involving either imbalanced or sparse topologies.}} Our approach relies on a dynamic communication graph with three Laplacian matrices (Figure \ref{Fig: 11}). The results are shown here, and the results using two Laplacian matrices are in Appendix~\ref{Appendix: 3}. Tables \ref{Tab:2} and Table \ref{Tab:10}(in Appendix~\ref{Appendix: 2}) demonstrate that D-PSGD-based AL-DSGD and MATCHA-based AL-DSGD are more robust to imbalanced topologies. To further evaluate AL-DSGD in communication-limited environments, we gradually reduce the communication graph's degree to simulate sparse topology (Figure~\ref{Fig: 12}) and compare AL-DSGD's performance with D-PSGD (Figure~\ref{Fig: 9}). AL-DSGD remains stable and robust to sparse topologies, as its performance does not significantly decrease until the degree is reduced to 38\%, while D-PSGD performs poorly even when the degree is only decreased to 84\%. 

\begin{table}[t!]
\centering
\resizebox{\columnwidth}{!}{%
\setlength{\tabcolsep}{0.35mm}{
\begin{tabular}{lllllllll}
\hline
\multicolumn{1}{c}{\textbf{}}       &\multicolumn{5}{c}{\textbf{TEST ACCURACY}}      \\ \hline
\multicolumn{1}{c}{\textbf{Algorithm}}   &\multicolumn{1}{c}{D=13}    &\multicolumn{1}{c}{D=11}    &\multicolumn{1}{c}{D=9}    &\multicolumn{1}{c}{D=7}    &\multicolumn{1}{c}{D=5}        \\ \hline
\multicolumn{1}{c}{\textbf{D-PSGD}}               &\multicolumn{1}{c}{91.18} &\multicolumn{1}{c}{91.21}   &\multicolumn{1}{c}{90.98}    &\multicolumn{1}{c}{90.26}    &\multicolumn{1}{c}{89.59}    \\ 
\multicolumn{1}{c}{\textbf{AL-DSGD-based D-PSGD}}               &\multicolumn{1}{c}{\textbf{93.68}} &\multicolumn{1}{c}{\textbf{93.59}}   &\multicolumn{1}{c}{\textbf{93.58}}    &\multicolumn{1}{c}{\textbf{93.30}}    &\multicolumn{1}{c}{\textbf{92.32}}   \\ \hline
\multicolumn{1}{c}{\textbf{MATCHA}}               &\multicolumn{1}{c}{93.65} &\multicolumn{1}{c}{\textbf{93.51}}   &\multicolumn{1}{c}{93.24}    &\multicolumn{1}{c}{92.86}    &\multicolumn{1}{c}{91.14}    \\  
\multicolumn{1}{c}{\textbf{AL-DSGD-based MATCHA}}              &\multicolumn{1}{c}{\textbf{93.94}} &\multicolumn{1}{c}{93.33}   &\multicolumn{1}{c}{\textbf{93.49}}    &\multicolumn{1}{c}{\textbf{93.30}}   &\multicolumn{1}{c}{\textbf{92.75}}    \\  \hline
\end{tabular}} 
}
\caption{\label{tab:widgets} The performance of D-PSGD, MATCHA, AL-DSGD-based D-PSGD, and AL-DSGD-based MATCHA for topolgies with different total degrees $D$. Results were obtained for CIFAR10 and ResNet-50. 
}
\label{Tab:7}
\end{table}

Finally, Table~\ref{Tab:7} includes the comparison of D-PSGD, MATCHA, and AL-DSGD. We applied AL-DSGD on the top of both the D-PSGD and MATCHA baselines and compared the results. Table~\ref{Tab:7} further confirms the claim that the AL-DSGD algorithm is highly robust to sparse topologies, as it consistently achieves better test accuracy compared to the baseline algorithms, D-PSGD and MATCHA, for nearly all the cases, particularly in sparse topology scenarios.
More details can be found in Appendix~\ref{Appendix: 2} \&~\ref{Appendix: 4}.

\section{Conclusions}
\label{sec:Conclusions}
This paper introduces Adjacent Leader Decentralized Gradient Descent (AL-DSGD), a novel decentralized distributed SGD algorithm. 
AL-DSGD assigns weights to neighboring learners based on their performance and degree for averaging and integrates corrective forces from best-performing and highest-degree neighbors during training. 
By employing a dynamic communication graph, AL-DSGD excels in communication-constrained scenarios, including imbalanced and sparse topologies. 
Theoretical proof of algorithm convergence is provided. Experimental results on various datasets and deep architectures demonstrate that AL-DSGD accelerates and stabilizes convergence of decentralized state-of-the-art techniques, improving test performance, especially in communication-constrained environments.



\bibliography{ecai-sample-and-instructions}

\newpage
\appendix
\onecolumn
\input{appendix2}

\end{document}

%% file: math_commands.tex
\usepackage{amsmath,amsfonts,bm}

\def\eqref#1{equation~\ref{#1}}

\def\1{\bm{1}}

\DeclareMathAlphabet{\mathsfit}{\encodingdefault}{\sfdefault}{m}{sl}
\SetMathAlphabet{\mathsfit}{bold}{\encodingdefault}{\sfdefault}{bx}{n}

\newcommand{\norm}[1]{\left\lVert#1\right\rVert}

%% file: appendix2.tex
\hrule height 4pt
\vskip 0.15in
\vskip -\parskip
\begin{center}
{\LARGE\bf Adjacent Leader Decentralized Stochastic Gradient Descent \par} 
\end{center}
\vskip 0.25in
\vskip -\parskip
\hrule height 1pt

\section{Dynamic Communication Graphs}\label{Appendix: 1}
In this appendix, we apply the ablation studies to explain the dynamic communication graphs method is neccessary in AL-DSGD algorithm. We choose pulling coefficients $\lambda_N = 0.1$ and $\lambda_{\tau} = 0.1$. The results can be found in figure~\ref{Fig: 10}. We set model weights coefficients $w_N = 0.1$ and $w_\tau = 0.1$. Without dynamic communication graphs, when applying addictive force according to loss performance, the worker with worse performance and smaller local training loss will affect other workers. The final test accuracy of the worker with less degree did not improve. The results indicate that without dynamic communication graphs, ALD-SGD, which only applies corrective forces, does not achieve better performance. This is because the workers with low degree overfit and negatively affect the others. Similar results are observed when only one corrective force or only the dynamic communication graph is applied.

\noindent\begin{minipage}{\textwidth}
    \centering
    \includegraphics[width=140mm]{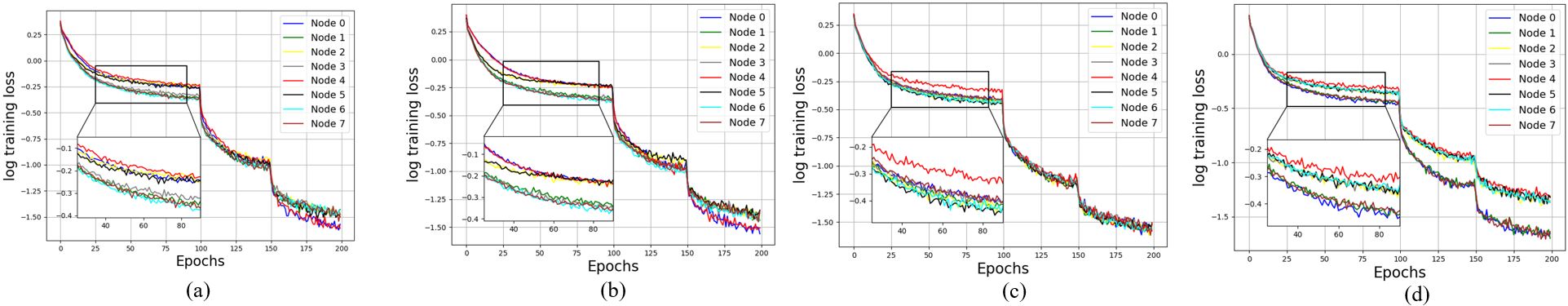}
    \centering
    \captionof{figure}{ResNet-50 model trained on CIFAR-10. Experiments to show AL-DSGD without dynamic communication graph. (a): D-PSGD. (b): D-PSGD based AL-DSGD without applying dynamic communication graph. (c):MATCHA. (d): MATCHA based AL-DSGD without applying dynamic communication graph.}
    \label{Fig: 10}
\end{minipage}


\section{Proof for Theorem \ref{thm:1}}\label{app:thm_1}
In this section, we are going to find a range of $\alpha$, and some averaging hyperparameter $\omega$, such that the spectral norm $\rho=\max\{\norm{\mathbb{E}\left[\widetilde{W}^{(k)}(I-J)\widetilde{W}^{(k)\intercal}\right]}, \norm{\mathbb{E}\left[\widetilde{W}^{(k)}\widetilde{W}^{(k)\intercal}\right]}\}$ is smaller than 1.

Recall the formula of mixing matrix $\widetilde{W}^{(k)}$:
\begin{align}\label{eq:W_tilde}
    \widetilde{W}^{(k)}&=(1-\omega_N-\omega_\tau)W^{(k)}+\omega_NA^{(k)N}+\omega_\tau A^{(k)\tau}.
\end{align}
Let $A^{(k)}=\frac{A^{(k)N}+A^{(k)\tau}}{2}$, $\omega=2\omega_N=2\omega_\tau$, we have
\begin{align}\label{eq:W_tilde2}
    \widetilde{W}^{(k)}&=(1-\omega)W^{(k)}+\omega A^{(k)},
\end{align}
where $A$ is still a left stochastic matrix. 
Therefore:
\begin{align}\label{eq:W1}
    \widetilde{W}^{(k)}(I-J)\widetilde{W}^{(k)\intercal}=\widetilde{W}^{(k)}\widetilde{W}^{(k)\intercal}-\widetilde{W}^{(k)}J\widetilde{W}^{(k)\intercal}
\end{align}
Since
\begin{align*}\label{eq:W2}
    \widetilde{W}^{(k)}\widetilde{W}^{(k)\intercal}=&\left[(1-\omega)W^{(k)}+\omega A^{(k)}\right]\left[(1-\omega)W^{(k)}+\omega A^{(k)}\right]^\intercal\notag\\
    =&(1-\omega)^2W^{(k)} W^{(k)\intercal}+\omega(1-\omega)W^{(k)} A^{(k)\intercal}+\omega(1-\omega)A^{(k)} W^{(k)\intercal}+\omega^2+\omega(1-\omega)A^{(k)} A^{(k)\intercal}\\
    \widetilde{W}^{(k)}J\widetilde{W}^{(k)\intercal}=&\left[(1-\omega)W^{(k)}J+\omega A^{(k)}J\right]\left[(1-\omega)W^{(k)}+\omega A^{(k)}\right]^\intercal\notag\\
    =&(1-\omega)^2W^{(k)}J W^{(k)\intercal}+\omega(1-\omega)W^{(k)}J A^{(k)\intercal}+\omega(1-\omega)A^{(k)}J W^{(k)\intercal}+\omega^2+\omega(1-\omega)A^{(k)} J A^{(k)\intercal}
\end{align*}
Since we know that $W^{(k)}$ is symmetric doubly stochastic matrix and $A^{(k)}$ is the left stochastic matrix, we know that $J=W^{(k)}J=JW^{(k)}$ and $J=JA^{(k)}\neq A^{(k)}J$. Putting (\ref{eq:W2}) back to (\ref{eq:W1}), we could get
\begin{align}
   \widetilde{W}^{(k)}(I-J)\widetilde{W}^{(k)\intercal}=&(1-\omega)^2\left[W^{(k)\intercal} W^{(k)}-J\right] +\omega(1-\omega)\left[W^{(k)} A^{(k)\intercal}-JA^{(k)\intercal}\right]\notag\\
   &+\omega(1-\omega)\left[A^{(k)} W^{(k)}-A^{(k)} J\right]+\omega^2\left[A^{(k)} A^{(k)\intercal}-A^{(k)} J A^{(k)\intercal}\right]
\end{align}
Therefore, we have
\begin{align}
   \norm{\mathbb{E}\left[\widetilde{W}^{(k)}(I-J)\widetilde{W}^{(k)\intercal}\right]}
   \leq&(1-\omega)^2\norm{\mathbb{E}\left[W^{(k)\intercal} W^{(k)}\right]-J}
   +2\omega(1-\omega)\norm{\mathbb{E}\left[W^{(k)} A^{(k)\intercal}-JA^{(k)\intercal}\right]}\notag\\
   &+\omega^2\norm{\mathbb{E}\left[A^{(k)} (I-J)A^{(k)\intercal}\right]}\label{eq:W_3}\\
   \norm{\mathbb{E}\left[\widetilde{W}^{(k)}\widetilde{W}^{(k)\intercal}\right]}
   \leq&(1-\omega)^2\norm{\mathbb{E}\left[W^{(k)\intercal} W^{(k)}\right]}
   +2\omega(1-\omega)\norm{\mathbb{E}\left[W^{(k)} A^{(k)\intercal}\right]}\notag\\
   &+\omega^2\norm{\mathbb{E}\left[A^{(k)}A^{(k)\intercal}\right]}\label{eq:W_32}
\end{align}

Firstly, We are going to bound each term in iequality (\ref{eq:W_3}) one by one.

\textbf{(1) Bound }$\norm{\mathbb{E}\left[W^{(k)\intercal} W^{(k)}\right]-J}$. 
\begin{align}\label{eq:norm_W_k}
    \norm{\mathbb{E}\left[W^{(k)\intercal} W^{(k)}\right]-J}=&\norm{\mathbb{E}\left[\left(I-\alpha L^{(k)}\right)^\intercal \left(I-\alpha L^{(k)}\right)\right]-J}\notag\\
    =&\norm{I-2\alpha\mathbb{E}\left[L^{(k)}\right]+\alpha^2\mathbb{E}\left[L^{(k)\intercal} L^{(k)}\right]-J}.
\end{align}
Recall there are two communication graph and $L^{(k)}$ is periodically switched between them:
$$ L^{(k)}=\left\{
\begin{aligned}
\sum_{j=1}^mB^{(k)}_{(1),j}L_{(1),j} \quad& \text{if $k$ mod $n = 1$}&\\
\sum_{j=1}^mB^{(k)}_{(2),j}L_{(2),j} \quad& \text{if $k$ mod $n = 2$}&\\
...&\\
\sum_{j=1}^mB^{(k)}_{(n),j}L_{(n),j} \quad& \text{if $k$ mod $n = 0$}&
\end{aligned}
\right.
$$
We analysis the case for $k$ mod $n=i$, where $i=1,2,...,n-1,0$. For notation convenience, we use $k$ mod $n=n$ instead of $k$ mod $n=0$ without loss of generality. Then the condition could be rewritten as $k$ mod $n=i$, where $i=1,2,...,n-1,n$. 

If $k$ mod $n=i$, then from Appendix B in \cite{wang2019matcha} we have
    \begin{align*} \mathbb{E}\left[L^{(k)}\right]&=\sum_{j=1}^mp_{(i),j}L_{(i),j}\\
    \mathbb{E}\left[L^{(k)\intercal} L^{(k)}\right]&=\left(\sum_{j=1}^mp_{(i),j}L_{(i),j}\right)^2+2\sum_{j=1}^Mp_{(i),j}(1-p_{(i),j})L_{(i),j}.
    \end{align*}
    And
    \begin{align}
    \norm{\mathbb{E}\left[W^{(k)\intercal} W^{(k)}\right]-J}\leq&\norm{\left(I-\alpha\sum_{j=1}^mp_{(i),j}L_{(i),j}\right)^2-J}+2\alpha^2\norm{2\sum_{j=1}^Mp_{(i),j}(1-p_{(i),j})L_{(i),j}}\notag\\
        =&\max\{(1-\alpha\lambda_{(i),2})^2,(1-\alpha\lambda_{i1),m})^2\}+2\alpha^2\zeta_{(i)},
    \end{align}
    where $\lambda_{(i),l}$ denote the $l$-th smallest eigenvalue of matrix $\sum_{j=1}^mp_{(i),j}L_{(i),j} $ and $\zeta_{(i)}>0$ denote the spectral norm of matrix $\sum_{j=1}^mp_{(i),j}(1-p_{(i),j})L_{(i),j}$.
    
In all, generalized all $k$ mod $n=i (i=1,...,n)$ we could conclude
\begin{align}       \norm{\mathbb{E}\left[W^{(k)\intercal} W^{(k)}\right]-J}
        \leq&\max\{(1-\alpha\lambda_{(1),2})^2,(1-\alpha\lambda_{(1),m})^2,...,(1-\alpha\lambda_{(n),2})^2,(1-\alpha\lambda_{(n),m})^2\}\notag\\
        &+2\alpha^2\max\{\zeta_{(1)},\zeta_{(2)},...,\zeta_{(n)}\}.
\end{align}

Assume $\lambda$ represents the eigenvalue such that $|1-\alpha\lambda|=\max\{|1-\alpha\lambda_{(1),2}|,|1-\alpha\lambda_{(1),m}|,|1-\alpha\lambda_{(2),2}|,|1-\alpha\lambda_{(2),m}|,...,|1-\alpha\lambda_{(n),2}|,|1-\alpha\lambda_{(n),m}|\}$, and $\zeta=\max\{\zeta_{(1)},\zeta_{(2)},...,\zeta_{(n)}\}$, we could have

\begin{align}       \norm{\mathbb{E}\left[W^{(k)\intercal} W^{(k)}\right]-J}\leq|1-\alpha\lambda|^2+2\alpha^2\zeta.
\end{align}

\textbf{(2) Bound} $\norm{\mathbb{E}\left[W^{(k)} A^{(k)\intercal}-JA^{(k)\intercal}\right]}$.

Because $A^{(k)}$ is a left stochastic matrix and $W^{(k)}$ is doubly stochastic matrix,  from the property of spectrum nor $\norm{\cdot}\leq\norm{\cdot}_1\norm{\cdot}_\infty$, we could know that for all $k$
$$\norm{A^{k}}\leq\sqrt{m},\quad \norm{W^{(k)}}\leq1.$$ Moreover, we could easy check $\norm{J}=1$. Therefore,
\begin{align}
    &\norm{\mathbb{E}\left[W^{(k)} A^{(k)\intercal}-JA^{(k)\intercal}\right]}\leq\norm{\mathbb{E}\left[(W^{(k)}-J)A^{(k)\intercal}\right]}
    \leq \mathbb{E}\left[\norm{(W^{(k)}-J)A^{(k)\intercal}}\right]\notag\\
    \leq& \mathbb{E}\left[\norm{(W^{(k)}-J)}\norm{A^{(k)\intercal}}\right]
    \leq \mathbb{E}\left[\left(\norm{W^{(k)}}+\norm{J}\right)\norm{A^{(k)\intercal}}\right]\leq2\sqrt{m}
\end{align}
\textbf{(3) Bound} $\norm{\mathbb{E}\left[A^{(k)} (I-J)A^{(k)\intercal}\right]}$.

\begin{align*}
    \norm{\mathbb{E}\left[A^{(k)} (I-J)A^{(k)\intercal}\right]}\leq \mathbb{E}\left[\norm{A^{(k)} (I-J)A^{(k)\intercal}}\right]\leq \mathbb{E}\left[\norm{A^{(k)}}^2\norm{(I-J)}\right]\leq m
\end{align*}

Combine \textbf{(1)-(3)} and (\ref{eq:W_3}), we have
\begin{align}\label{eq:W_4}
   \norm{\mathbb{E}\left[\widetilde{W}^{(k)}(I-J)\widetilde{W}^{(k)\intercal}\right]}
   \leq&(1-\omega)^2(1-\alpha\lambda)^2
   +4\omega(1-\omega)\sqrt{m}+2\omega^2m
\end{align}

Similarly, we are going to bound each term in iequality (\ref{eq:W_32}) one by one as well.

\textbf{(4) Bound }$\norm{\mathbb{E}\left[W^{(k)\intercal} W^{(k)}\right]}$. 

Similar to proof in \textbf{(1)}, if $k$ mod $n=i$ ($i=1,...,n$), we have
    \begin{align}
    \norm{\mathbb{E}\left[W^{(k)\intercal} W^{(k)}\right]}\leq&\norm{\left(I-\alpha\sum_{j=1}^mp_{(i),j}L_{(i),j}\right)^2}+2\alpha^2\norm{2\sum_{j=1}^Mp_{(i),j}(1-p_{(i),j})L_{(i),j}}\notag\\
        =&\max\{(1-\alpha\lambda_{(i),2})^2,(1-\alpha\lambda_{(i),m})^2\}+2\alpha^2\zeta_{(i)},
    \end{align}
    where $\lambda_{(i),l}$ denote the $l$-th smallest eigenvalue of matrix $\sum_{j=1}^mp_{(i),j}L_{(i),j} $ and $\zeta_{(1)}>0$ denote the spectral norm of matrix $\sum_{j=1}^mp_{(i),j}(1-p_{(i),j})L_{(i),j}$.

    

In all, generalize all $k$ mod $n=i (i=1,...,n)$ we could conclude
\begin{align}       \norm{\mathbb{E}\left[W^{(k)\intercal} W^{(k)}\right]}
        \leq&\max\{(1-\alpha\lambda_{(1),2})^2,(1-\alpha\lambda_{(1),m})^2,...,(1-\alpha\lambda_{(n),2})^2,(1-\alpha\lambda_{(n),m})^2\}\notag\\
        &+2\alpha^2\max\{\zeta_{(1)},\zeta_{(2)},...,\zeta_{(n)}\}.
\end{align}

Assume $\lambda$ represents the eigenvalue such that $|1-\alpha\lambda|=\max\{|1-\alpha\lambda_{(1),2}|,|1-\alpha\lambda_{(1),m}|,|1-\alpha\lambda_{(2),2}|,|1-\alpha\lambda_{(2),m}|,...,|1-\alpha\lambda_{(n),2}|,|1-\alpha\lambda_{(n),m}|\}$, and $\zeta=\max\{\zeta_{(1)},\zeta_{(2)},...,\zeta_{(n)}\}$, we could have

\begin{align}       \norm{\mathbb{E}\left[W^{(k)\intercal} W^{(k)}\right]}\leq|1-\alpha\lambda|^2+2\alpha^2\zeta.
\end{align}

\textbf{(2) Bound} $\norm{\mathbb{E}\left[W^{(k)} A^{(k)\intercal}\right]}$.

Because $A^{(k)}$ is a left stochastic matrix and $W^{(k)}$ is doubly stochastic matrix,  from the property of spectrum nor $\norm{\cdot}\leq\norm{\cdot}_1\norm{\cdot}_\infty$, we could know that for all $k$
$$\norm{A^{k}}\leq\sqrt{m},\quad \norm{W^{(k)}}\leq1.$$ Moreover, we could easy check $\norm{J}=1$. Therefore,
\begin{align}
    &\norm{\mathbb{E}\left[W^{(k)} A^{(k)\intercal}\right]}\leq \mathbb{E}\left[\norm{W^{(k)}}\norm{A^{(k)\intercal}}\right]\leq\sqrt{m}
\end{align}
\textbf{(3) Bound} $\norm{\mathbb{E}\left[A^{(k)}A^{(k)\intercal}\right]}$.

\begin{align*}
    \norm{\mathbb{E}\left[A^{(k)}A^{(k)\intercal}\right]}\leq \mathbb{E}\left[\norm{A^{(k)}}^2\right]\leq m
\end{align*}

Combine \textbf{(4)-(5)} and (\ref{eq:W_32}), we have
\begin{align}\label{eq:W_42}
   \norm{\mathbb{E}\left[\widetilde{W}^{(k)}\widetilde{W}^{(k)\intercal}\right]}
   \leq&(1-\omega)^2(1-\alpha\lambda)^2
   +2\omega(1-\omega)\sqrt{m}+\omega^2m
\end{align}

From the proof in Appendix B of \cite{wang2019matcha}, we know that $\lambda>0$. We assume $0<\alpha<\frac{1}{\lambda}$, and $\omega\in(0,1)$, combine (\ref{eq:W_4}) and (\ref{eq:W_42}) we have
\begin{align}\label{eq:W5}
   \rho=&\max\{\norm{\mathbb{E}\left[\widetilde{W}^{(k)}(I-J)\widetilde{W}^{(k)\intercal}\right]}, \norm{\mathbb{E}\left[\widetilde{W}^{(k)}\widetilde{W}^{(k)\intercal}\right]}\}\notag\\
   \leq&(1-\omega)^2(1-\alpha\lambda)^2
   +4\omega(1-\omega)\sqrt{m}+2\omega^2m+2\alpha^2\zeta\notag\\
   \leq&(1-\omega)^2(1-\alpha\lambda)^2
   +4\omega(1-\omega)\sqrt{m}+4\omega^2m+2\alpha^2\zeta\notag\\
   =&(1-\omega)^2(1-\alpha\lambda)^2+4\omega(1-\omega)\sqrt{m}\left[(1-\alpha\lambda)+\alpha\lambda\right]+\omega^2m+2\alpha^2\zeta\notag\\
   \leq&(1-\omega)^2(1-\alpha\lambda)^2+4\omega(1-\omega)\sqrt{m}(1-\alpha\lambda)+4\omega^2m+4\alpha\lambda\omega(1-\omega)\sqrt{m}+2\alpha^2\zeta\notag\\
   \leq&\left[(1-\omega)(1-\alpha\lambda)+2\sqrt{m}\omega\right]^2+4\alpha\lambda\omega(1-\omega)\sqrt{m}+2\alpha^2\zeta\notag\\
   \leq&\left[(1-\omega)(1-\alpha\lambda)+2\sqrt{m}\omega\right]^2+4\omega\sqrt{m}+2\alpha^2\zeta
\end{align}
Define $f_{\lambda,\alpha}(\omega)=\left[(1-\omega)(1-\alpha\lambda)+2\sqrt{m}\omega\right]^2+4\omega\sqrt{m}+2\alpha^2\zeta$, we have
$$f'_{\lambda,\alpha}(\omega)=2\left[(1-\omega)(1-\alpha\lambda)+2\sqrt{m}\omega\right]\left[2\sqrt{m}\omega-(1-\alpha\lambda)\right]+4\sqrt{m}.$$
$m$ is the number of the worker and it must satisfy $m>1$. Together with $\alpha\lambda\in(0,1)$, we could conclude $f'_{\lambda,\alpha}(\omega)>0$ for all $\omega\in(0,1)$. Then take 
$$\omega_0=\frac{1-\alpha\lambda}{2k\sqrt{m}},$$
where $k>1$. We know that for all $\omega\in(0,\omega_0)$,
\begin{align}
    f_{\lambda,\alpha}(\omega)\leq f_{\lambda,\alpha}(\omega_0)
    =&\left[\left(1-\frac{1-\alpha\lambda}{2k\sqrt{m}}\right)(1-\alpha\lambda)+\frac{1-\alpha\lambda}{k}\right]^2+\frac{2(1-\alpha\lambda)}{k}+2\alpha^2\zeta\notag\\
    \leq& \frac{(k+1)^2}{k^2}(1-\alpha\lambda)^2+\frac{2(1-\alpha\lambda)}{k}+2\alpha^2\zeta
\end{align}
Define $h_\lambda(\alpha)=\frac{(k+1)^2}{k^2}(1-\alpha\lambda)^2+\frac{2(1-\alpha\lambda)}{k}+2\alpha^2\zeta$, then we have
\begin{align*}
    h'_{\lambda}(\alpha)&=-\frac{2(k+1)^2}{k^2}\lambda(1-\alpha\lambda)-\frac{2\lambda}{k}+4\alpha\zeta,\\
    h''_{\lambda}(\alpha)&=\frac{2(k+1)^2}{k^2}\lambda^2+4\zeta.
\end{align*}
Since $h''_\lambda(\alpha)>0$, $h_\lambda(\alpha)$ is convex quadratic fucntion. Let $h'_\lambda(\alpha)=0$, we could get the minimun point is:
$$\alpha^*=\frac{[(k+1)^2+k]\lambda}{(k+1)^2\lambda^2+2k^2\zeta}.$$
We take $\widetilde{\alpha}=\frac{(k+1)^2\lambda}{(k+1)^2\lambda^2+2k^2\zeta},$ it is easy to know $0<\widetilde{\alpha}<\alpha^*$.
\begin{align*}
    h_\lambda(\widetilde{\alpha})=\frac{(k+1)^2}{k^2}(1-\widetilde{\alpha}\lambda)^2+\frac{2}{k}(1-\widetilde{\alpha}\lambda)+2\alpha^2\zeta=\frac{4k\zeta+2(k+1)^2\zeta}{(k+1)^2\lambda^2+2k^2\zeta}.
\end{align*}

It is obvious that $\widetilde{\alpha}\lambda\in(0,1)$. Then, we are going to compute the bound for $k$ to ensure $h_\lambda(\widetilde{\alpha})<1$.When $k>\max\{1,\frac{8\zeta}{\lambda^2}-1\}$, we have:
\begin{align*}
    &\frac{k+1}{8}>\frac{\zeta}{\lambda^2}\Rightarrow \frac{(k+1)^2}{8K+8}>\frac{\zeta}{\lambda^2}\Rightarrow \frac{(k+1)^2}{8K+4}>\frac{\zeta}{\lambda^2}\\
    \Rightarrow&2(k+1)^2\zeta+4k\zeta<(k+1)^2\lambda^2+2k^2\zeta\notag\\
    \Rightarrow&\frac{4k\zeta+2(k+1)^2\zeta}{(k+1)^2\lambda^2+2k^2\zeta}<1
\end{align*}

For any $k>\max\{1,\frac{8\zeta}{\lambda^2}-1\}$, by the convex property of $h_{\lambda}(\alpha)$, we know when $\alpha\in\left(\alpha_{min}, \alpha_{max}\right)$, where:
\begin{align*}
    \alpha_{min} = \frac{(k+1)^2\lambda}{(k+1)^2\lambda^2+2k^2\zeta},\quad
    \alpha_{max} = \min\{\frac{1}{\lambda}, \frac{[(k+1)^2+k]\lambda}{(k+1)^2\lambda^2+2k^2\zeta}\}.
\end{align*}
There exists a range of averaging parameter
$\omega\in(0,\frac{1-\alpha\lambda}{2k\sqrt{m}})$, such that
$$\rho=\max\{\norm{\mathbb{E}\left[\widetilde{W}^{(k)}(I-J)\widetilde{W}^{(k)\intercal}\right]}, \norm{\mathbb{E}\left[\widetilde{W}^{(k)}\widetilde{W}^{(k)\intercal}\right]}\}\leq h_\lambda(\alpha)<1$$.

Furthermor, $k>\frac{\lambda^2}{2\zeta}\Rightarrow\frac{1}{\lambda}>\frac{[(k+1)^2+k]\lambda}{(k+1)^2\lambda^2+2k^2\zeta}$. Therefore, for any $k>\max\{1,\frac{8\zeta}{\lambda^2}-1,\frac{\lambda^2}{2\zeta}\}$, when $\alpha\in\left(\alpha_{min}, \alpha_{max}\right)$, where:
\begin{align*}
    \alpha_{min} = \frac{(k+1)^2\lambda}{(k+1)^2\lambda^2+2k^2\zeta},\quad
    \alpha_{max} = \frac{[(k+1)^2+k]\lambda}{(k+1)^2\lambda^2+2k^2\zeta}.
\end{align*}

Going back to the assumption for $\lambda$, when $1-\alpha\lambda\in(0,1)$ always holds for sufficient small $\alpha$, $\lambda$ represents the eigenvalue such that $|1-\alpha\lambda|=\max\{|1-\alpha\lambda_{(1),2}|,|1-\alpha\lambda_{(1),m}|,|1-\alpha\lambda_{(2),2}|,|1-\alpha\lambda_{(2),m}|,...,|1-\alpha\lambda_{(n),2}|,|1-\alpha\lambda_{(n),m}|\}$ should be exactly $\lambda=\min\{\lambda_{(1),2},\lambda_{(2),2},...,\lambda_{(n),2}\}$.

We generalized the above analysis of the construction for $\alpha$ and $\omega$ as the following. Assume $\lambda_{min}=\min\{\lambda_{(i),2}:i=1,...,n\}$, $\lambda_{max}=\min\{\lambda_{(i),m}|i=1,...,n\}$ and $\zeta=\max\{\zeta_{(1)},\zeta_{(2)},...,\zeta_{(n)}\}$. For any $k>\max\{1,\frac{8\zeta}{\lambda_{min}^2}-1,\frac{\lambda_{max}^2}{2\zeta}\}$, there exists a range $\alpha\in(\frac{(k+1)^2\lambda_{min}}{(k+1)^2\lambda_{min}^2+2k^2\zeta}, \frac{[(k+1)^2+k]\lambda_{min}}{(k+1)^2\lambda_{min}^2+2k^2\zeta})$, such that for any $\alpha$ in this range, we could find a range $\omega\in(0,\frac{1-\alpha\lambda}{2k\sqrt{m}})$ such that the spectral norm 
$$\rho=\max\{\norm{\mathbb{E}\left[\widetilde{W}^{(k)}(I-J)\widetilde{W}^{(k)\intercal}\right]}, \norm{\mathbb{E}\left[\widetilde{W}^{(k)}\widetilde{W}^{(k)\intercal}\right]}\}<1$$.

\section{Proof for Theorem \ref{thm:2}}

Before moving into the detailed proof, we first introduce a lemma.
\begin{lemma}\label{lma:1}
    Let $\{\widetilde{W}^{(k)}\}_{k=1}^{\infty}$ be i.i.d matrix generated from AL-DSGD algorithm $$\rho=\max\{\norm{\mathbb{E}\left[\widetilde{W}^{(k)}(I-J)\widetilde{W}^{(k)\intercal}\right]}, \norm{\mathbb{E}\left[\widetilde{W}^{(k)}\widetilde{W}^{(k)\intercal}\right]}\}<1,$$
    Then we could claim
    \begin{align*}
        \mathbb{E}\left[\norm{B\left(\prod_{k=1}^n\widetilde{W}^{(k)}\right)(I-J)}_F^2\right]\leq\rho^n\norm{B}_F^2.
    \end{align*}
\end{lemma}
\begin{proof}
    See Appendix \ref{supp:lma1}
\end{proof}

Recall the update rule for AL-DSGD algorithm:
\begin{align}
    X_{k+1}=\widetilde{W}^{k}X_{k}-\gamma(1-\omega)Diag\left(W^{(k)}\right)\left[G^{(k)}+\lambda_N(X_{k}-X_k^N)+\lambda_\tau(X_{k}-X_k^\tau)\right],
\end{align}
where 
\begin{align*}
    X_k &= [x_{1,k},...,x_{m,k}],\\
    G^{(k)} &= [g_1(x_{1,k}),...,g_m(x_{m,k})],\\
    \nabla F^{(k)}&=[\nabla F_1(x_{1,k}),...,\nabla F_m(x_{m,k})]
\end{align*}
Let $\lambda=2\lambda_N=2\lambda_\tau$, and $C_k = \frac{1}{2}(X_k^N+X_k^\tau)$. Then we have
\begin{align}
    X_{k+1}=\widetilde{W}^{k}X_{k}-\gamma(1-\omega)Diag\left(W^{(k)}\right)\left[G^{(k)}+\lambda(X_{k}-C_k)\right]
\end{align}
By the construction of $W^{(k)}$, the diagonal term in $W^{(k)}$ are all $1-\alpha$, we have
\begin{align}
    X_{k+1}=\widetilde{W}^{k}X_{k}-\gamma(1-\omega)(1-\alpha)\left[G^{(k)}+\lambda(X_{k}-C_k)\right]
\end{align}
After taking the average and define $\eta = \gamma(1-\omega)(1-\alpha)$, we have
\begin{align}
    \overline{x}_{k+1}=\overline{x}_k-\eta\left[\frac{G^{(k)}\mathbf{1}}{m}+\lambda(\overline{x}_{k+\frac{1}{2}}-\overline{c}_k)\right]=\overline{x}_k-\eta\left[\frac{G^{(k)}\mathbf{1}}{m}+\lambda\Delta^{(k)}\right].
\end{align}

Denote $\Delta^{(k)} = \overline{x}_{k+\frac{1}{2}}-\overline{c}_k$. 
from Assumption \ref{asm:1} (5), we could conclude $\norm{\Delta^{(k)}}^2\leq\Delta^2$.

Then we have
\begin{align}
    F(\overline{x}_{k+1})-F(\overline{x}_{k})\leq& \langle\nabla F(\overline{x}_k), \overline{x}_{k+1}-\overline{x}_k\rangle+\frac{L}{2}\norm{\overline{x}_{k+1}-\overline{x}_k}\notag\\
    \leq&-\eta\langle\nabla F(\overline{x}_k), \frac{G^{(k)}\mathbf{1}}{m}+\lambda\Delta^{(k)}\rangle+\frac{\eta^2L}{2}\norm{\frac{G^{(k)}\mathbf{1}}{m}+\lambda\Delta^{(k)}}\notag\\
    \leq&-\eta\langle\nabla F(\overline{x}_k), \frac{G^{(k)}\mathbf{1}}{m}\rangle+\frac{\eta^2L}{2}\norm{\frac{G^{(k)}\mathbf{1}}{m}}\notag\\
    &-\lambda\eta\langle\nabla F(\overline{x}_k),\Delta^{(k)}\rangle+\frac{\lambda^2\eta^2L}{2}\norm{\Delta^{(k)}}^2+\lambda\eta^2L\langle\frac{G^{(k)}\mathbf{1}}{m},\Delta^{(k)}\rangle\notag\\
    \leq&-\eta\langle\nabla F(\overline{x}_k), \frac{G^{(k)}\mathbf{1}}{m}\rangle+\frac{\eta^2L}{2}\norm{\frac{G^{(k)}\mathbf{1}}{m}}+\lambda\eta\beta\Delta+\lambda\eta^2L\beta \Delta+\frac{\lambda^2\eta^2L\Delta^2}{2}
\end{align}

Inspired by the proof in Appendix C.3 of \cite{wang2019matcha}, we have:
\begin{align}
    \mathbb{E}\left[F(\overline{x}_{k+1})-F(\overline{x}_{k})\right]\leq&-\frac{\eta}{2}\mathbb{E}\left[\norm{\nabla F(\overline{x}_k)}^2\right]-\frac{\eta}{2}(1-\eta L)\mathbb{E}\left[\norm{\frac{\nabla F^{(k)}\mathbf{1}}{m}}^2\right]\\
    &+\frac{\eta L^2}{2m}\mathbb{E}\left[\norm{X_k(I-J)}_F^2\right]+\frac{\eta^2L\sigma^2}{2m}+\lambda\eta\beta\Delta+\lambda\eta^2L\beta \Delta+\frac{\lambda^2\eta^2L\Delta^2}{2}
\end{align}

Denote $M=\frac{\eta^2L\sigma^2}{2m}+\lambda\eta\beta\Delta+\lambda\eta^2L\beta \Delta+\frac{\lambda^2\eta^2L\Delta^2}{2}$, the bound could be simplified as
\begin{align}
    \mathbb{E}\left[F(\overline{x}_{k+1})-F(\overline{x}_{k})\right]\leq-\frac{\eta}{2}\mathbb{E}\left[\norm{\nabla F(\overline{x}_k)}^2\right]-\frac{\eta}{2}(1-\eta L)\mathbb{E}\left[\norm{\frac{\nabla F^{(k)}\mathbf{1}}{m}}^2\right]+\frac{\eta L^2}{2m}\mathbb{E}\left[\norm{X_k(I-J)}_F^2\right]+M.
\end{align}
Summing over all iterations and then take the average, we have
\begin{align}
    \frac{\mathbb{E}\left[F(\overline{x}_{K})-F(\overline{x}_{1})\right]}{K}
    \leq& -\frac{\eta}{2}\frac{1}{K}\sum_{k=1}^k\mathbb{E}\left[\norm{\nabla F(\overline{x}_k)}^2\right]-\frac{\eta}{2}(1-\eta L)\frac{1}{K}\sum_{k=1}^k\mathbb{E}\left[\norm{\frac{\nabla F^{(k)}\mathbf{1}}{m}}^2\right]\notag\\
    &+\frac{\eta L^2}{2mK}\sum_{k=1}^k\mathbb{E}\left[\norm{X_k(I-J)}_F^2\right]+M.
\end{align}
By rearranging the inequality, we have
\begin{align}\label{ieq:sumF1}
    \frac{1}{K}\sum_{k=1}^k\mathbb{E}\left[\norm{\nabla F(\overline{x}_k)}^2\right]
    \leq&\frac{2(F(\overline{x}_{1})-F^*)}{\eta K}-(1-\eta L)\frac{1}{K}\sum_{k=1}^k\mathbb{E}\left[\norm{\frac{\nabla F^{(k)}\mathbf{1}}{m}}^2\right]\notag\\
    &+\frac{L^2}{mK}\sum_{k=1}^k\mathbb{E}\left[\norm{X_k(I-J)}_F^2\right]+\frac{2M}{\eta}\notag\\
    \leq&\frac{2(F(\overline{x}_{1})-F^*)}{\eta K}+\frac{L^2}{mK}\sum_{k=1}^k\mathbb{E}\left[\norm{X_k(I-J)}_F^2\right]+\frac{2M}{\eta}.
\end{align}
Then we are goint to bound $\mathbb{E}\left[\norm{X_k(I-J)}_F^2\right]$. By the property of matrix $J$, we have
\begin{align}
    X_k(I-J)=&(X_{k-1}-\eta(G^{(k-1)}+\lambda\Delta_{k-1})\widetilde{W}^{(k-1)}(I-J)\notag\\
    =&X_{k-1}\widetilde{W}^{(k-1)}(I-J)-\eta(G^{(k-1)}+\lambda\Delta^{(k-1)})\widetilde{W}^{(k-1)}(I-J)\notag\\
    =&...\notag\\
    =&X_{1}\prod_{q=1}^{k-1}\widetilde{W}^{(q)}(I-J)-\eta \sum_{q=1}^{k-1}(G^{(k-1)}+\lambda\Delta^{(k)})\left(\prod_{l=q}^{k-1}\widetilde{W}^{(l)}\right)(I-J)
\end{align}
Without loss of generalizty, assume $X_{1}=0$. Therefore, by Assumption (\ref{asm:1}) and Lemma \ref{lma:1} we have
\begin{align}
    &\norm{X_k(I-J)}_F^2\notag\\
    =&\eta^2\norm{\sum_{q=1}^{k-1}(G^{(k-1)}+\lambda\Delta^{(k)})\left(\prod_{l=q}^{k-1}\widetilde{W}^{(l)}\right)(I-J)}_F^2\notag\\
    \leq&2\eta^2\norm{\sum_{q=1}^{k-1}G^{(k-1)}\left(\prod_{l=q}^{k-1}\widetilde{W}^{(l)}\right)(I-J)}_F^2+2\eta^2\lambda^2\norm{\sum_{q=1}^{k-1}\Delta^{(k)}\left(\prod_{l=q}^{k-1}\widetilde{W}^{(l)}\right)(I-J)}_F^2.
\end{align}
From Assumption \ref{asm:1} (5), we could conclude $\norm{\Delta^{(k)}}^2\leq\Delta^2$ for all $k$. Combine with Lemma \ref{lma:1}, we have
\begin{align}
    &\norm{X_k(I-J)}_F^2\notag\\
    \leq&2\eta^2\norm{\sum_{q=1}^{k-1}G^{(k-1)}\left(\prod_{l=q}^{k-1}\widetilde{W}^{(l)}\right)(I-J)}_F^2+2\eta^2\lambda^2\Delta^2\sum_{q=1}^k\rho^q\notag\\
    \leq&2\eta^2\norm{\sum_{q=1}^{k-1}G^{(k-1)}\left(\prod_{l=q}^{k-1}\widetilde{W}^{(l)}\right)(I-J)}_F^2+\frac{2\eta^2\lambda^2\Delta^2\rho}{1-\rho}\notag\\
    \leq&2\eta^2\norm{\sum_{q=1}^{k-1}\left(G^{(k-1)}-\nabla F^{(q)}\right)\left(\prod_{l=q}^{k-1}\widetilde{W}^{(l)}\right)(I-J)}_F^2+2\eta^2\norm{\sum_{q=1}^{k-1}\nabla F^{(q)}\left(\prod_{l=q}^{k-1}\widetilde{W}^{(l)}\right)(I-J)}_F^2+\frac{2\eta^2\lambda^2\Delta^2\rho}{1-\rho}.
\end{align}
Taking expectation, we have:
\begin{align}
    \mathbb{E}\left[\norm{X_k(I-J)}_F^2\right]\leq 2\eta^2\mathbb{E}\left[\norm{\sum_{q=1}^{k-1}\nabla F^{(q)}\left(\prod_{l=q}^{k-1}\widetilde{W}^{(l)}\right)(I-J)}_F^2\right] +\frac{2m\eta^2\sigma^2\rho}{1-\rho}+\frac{2\eta^2\lambda^2\Delta^2\rho}{1-\rho}
\end{align}
For notation simplicity, let $B_{q,p}=\left(\prod_{l=q}^p\widetilde{W}^{(l)}\right)(I-J)$, then we have
\begin{align}
    &\mathbb{E}\left[\norm{\sum_{q=1}^{k-1}\nabla F^{(q)}\left(\prod_{l=q}^{k-1}\widetilde{W}^{(l)}\right)(I-J)}_F^2\right]\notag\\
    =&\sum_{q=1}^{k-1}\mathbb{E}\left[\norm{\nabla F^{(q)}B_{q,k-1}}_F^2\right]+\sum_{q=1}^{k-1}\sum_{p=1,p\neq q}^{k-1}\mathbb{E}\left[Tr\{B_{q,k-1}^\intercal\nabla F^{(q)\intercal}\nabla F^{(p)}B_{p,k-1}\}\right]\notag\\
    \leq& \sum_{q=1}^{k-1}\mathbb{E}\left[\norm{\nabla F^{(q)}B_{q,k-1}}_F^2\right]+\sum_{q=1}^{k-1}\sum_{p=1,p\neq q}^{k-1}\mathbb{E}\left[\norm{\nabla F^{(q)}B_{q,k-1}}_F^2\norm{\nabla F^{(p)}B_{p,k-1}}_F^2\right]\notag\\
    \leq&\sum_{q=1}^{k-1}\mathbb{E}\left[\norm{\nabla F^{(q)}B_{q,k-1}}_F^2\right]+\sum_{q=1}^{k-1}\sum_{p=1,p\neq q}^{k-1}\mathbb{E}\left[\frac{1}{2\epsilon}\norm{\nabla F^{(q)}B_{q,k-1}}_F^2+\frac{\epsilon}{2}\norm{\nabla F^{(p)}B_{p,k-1}}_F^2\right]\notag\\
    \leq&\sum_{q=1}^{k-1}\mathbb{E}\left[\norm{\nabla F^{(q)}B_{q,k-1}}_F^2\right]+\sum_{q=1}^{k-1}\sum_{p=1,p\neq q}^{k-1}\mathbb{E}\left[\frac{\rho^{k-q}}{2\epsilon}\norm{\nabla F^{(q)}}_F^2+\frac{\rho^{k-p}\epsilon}{2}\norm{\nabla F^{(p)}}_F^2\right].
\end{align}
Taking $\epsilon=\rho^{\frac{p-q}{2}}$, we have
\begin{align*}
    &\mathbb{E}\left[\norm{\sum_{q=1}^{k-1}\nabla F^{(q)}\left(\prod_{l=q}^{k-1}\widetilde{W}^{(l)}\right)(I-J)}_F^2\right]\notag\\
    \leq&\sum_{q=1}^{k-1}\mathbb{E}\left[\norm{\nabla F^{(q)}B_{q,k-1}}_F^2\right]+\frac{1}{2}\sum_{q=1}^{k-1}\sum_{p=1,p\neq q}^{k-1}\sqrt{\rho}^{2k-p-q}\mathbb{E}\left[\norm{\nabla F^{(q)}}_F^2+\norm{\nabla F^{(p)}}_F^2\right]\notag\\
    =&\sum_{q=1}^{k-1}\mathbb{E}\left[\norm{\nabla F^{(q)}B_{q,k-1}}_F^2\right]+\sum_{q=1}^{k-1}\sqrt{\rho}^{k-q}\mathbb{E}\left[\norm{\nabla F^{(q)}}_F^2\right]\sum_{p=1,p\neq q}^{k-1}\sqrt{\rho}^{k-p}\notag\\
    =&\sum_{q=1}^{k-1}\mathbb{E}\left[\norm{\nabla F^{(q)}B_{q,k-1}}_F^2\right]+\sum_{q=1}^{k-1}\sqrt{\rho}^{k-q}\mathbb{E}\left[\norm{\nabla F^{(q)}}_F^2\right]\left(\sum_{p=1}^{k-1}\sqrt{\rho}^{k-p}-\sqrt{\rho}^{k-q}\right)\notag\\
    \leq&\frac{\sqrt{\rho}}{1-\sqrt{\rho}}\sum_{q=1}^{k-1}\sqrt{\rho}^{k-q}\mathbb{E}\left[\norm{\nabla F^{(q)}}_F^2\right].
\end{align*}
Therefore,
\begin{align}
    \mathbb{E}\left[\norm{X_k(I-J)}_F^2\right]\leq\frac{2\eta^2\sqrt{\rho}}{1-\sqrt{\rho}}\sum_{q=1}^{k-1}\sqrt{\rho}^{k-q}\mathbb{E}\left[\norm{\nabla F^{(q)}}_F^2\right] +\frac{2\eta^2\rho(m\sigma^2+\lambda^2\Delta^2)}{1-\rho}.
\end{align}
Therefore
\begin{align}
    \frac{1}{mK}\sum_{k=1}^K\mathbb{E}\left[\norm{X_k(I-J)}_F^2\right]\leq&\frac{2\eta^2\sqrt{\rho}}{mK(1-\sqrt{\rho})}\sum_{k=1}^K\sum_{q=1}^{k-1}\sqrt{\rho}^{k-q}\mathbb{E}\left[\norm{\nabla F^{(q)}}_F^2\right] +\frac{2\eta^2\rho(m\sigma^2+\lambda^2\Delta^2)}{m(1-\rho)}\notag\\
    \leq&\frac{2\eta^2\sqrt{\rho}}{mK(1-\sqrt{\rho})}\sum_{k=1}^K\frac{\sqrt{\rho}}{1-\sqrt{\rho}}\mathbb{E}\left[\norm{\nabla F^{(q)}}_F^2\right] +\frac{2\eta^2\rho(m\sigma^2+\lambda^2\Delta^2)}{m(1-\rho)}\notag\\
    =&\frac{2\eta^2\rho}{mK(1-\sqrt{\rho})^2}\sum_{k=1}^K\mathbb{E}\left[\norm{\nabla F^{(q)}}_F^2\right] +\frac{2\eta^2\rho(m\sigma^2+\lambda^2\Delta^2)}{m(1-\rho)}
\end{align}
Note that
\begin{align}
    \norm{\nabla F^{(q)}}_F^2=&\sum_{i=1}^m\norm{\nabla F_i(x_{i,k})}^2\notag\\
    =&\sum_{i=1}^m\norm{\nabla F_i(x_{i,k})-\nabla F(x_{i,k})+\nabla F(x_{i,k})-\nabla F(\overline{x}_k)+\nabla F(\overline{x}_k)}^2\notag\\
    \leq&3\sum_{i=1}^m\left[\norm{\nabla F_i(x_{i,k})-\nabla F(x_{i,k})}^2+\norm{\nabla F(x_{i,k})-\nabla F(\overline{x}_k)}^2+\norm{\nabla F(\overline{x}_k)}^2\right]\notag\\
    \leq&3m\zeta^2+3L^2\norm{X_k(I-J)}_F^2+3m\norm{\nabla F(\overline{x}_k)}^2.
\end{align}
Therefore, we have
\begin{align}
    \frac{1}{mK}\sum_{k=1}^K\mathbb{E}\left[\norm{X_k(I-J)}_F^2\right]\leq&\frac{2\eta^2\rho(m\sigma^2+\lambda^2\Delta^2)}{m(1-\rho)}+\frac{6\eta^2\zeta^2\rho}{(1-\sqrt{\rho})^2}+\frac{6\eta^2\zeta^2\rho}{(1-\sqrt{\rho})^2}\frac{1}{mK}\mathbb{E}\left[\norm{X_k(I-J)}_F^2\right]\notag\\
    &+\frac{6\eta^2\rho}{(1-\sqrt{\rho})^2}\frac{1}{K}\sum_{i=1}^K\mathbb{E}\left[\norm{\nabla F(\overline{x}_k)}^2\right].
\end{align}
Define $D=\frac{6\eta^2L^2\rho}{(1-\sqrt{\rho})^2}$, by rearranging we have
\begin{align}\label{ieq:XIJ}
    \frac{1}{mK}\sum_{k=1}^K\mathbb{E}\left[\norm{X_k(I-J)}_F^2\right]\leq\frac{1}{1-2D}\left[\frac{2\eta^2\rho(m\sigma^2+\lambda^2\Delta^2)}{m(1-\rho)}+\frac{6\eta^2\zeta^2\rho}{(1-\sqrt{\rho})^2}+\frac{6\eta^2\rho}{(1-\sqrt{\rho})^2}\frac{1}{K}\sum_{i=1}^K\mathbb{E}\left[\norm{\nabla F(\overline{x}_k)}^2\right]\right]
\end{align}
Plugging (\ref{ieq:XIJ}) back to (\ref{ieq:sumF1}), we have
\begin{align}
\frac{1}{K} \sum_{i=1}^K\mathbb{E}\left[\norm{\nabla F(\overline{x}_k)}^2\right]\leq &\frac{2(F(\overline{x}_{1})-F^*)}{\eta K}+\frac{2M}{\eta}+\frac{1}{1-D}\frac{2\eta^2L^2\rho(m\sigma^2+\lambda^2\Delta^2)}{m(1-\rho)}+\frac{D\zeta^2}{1-D}\notag\\
&+\frac{D}{1-D}\frac{1}{K} \sum_{i=1}^K\mathbb{E}\left[\norm{\nabla F(\overline{x}_k)}^2\right].
\end{align}
Therefore,
\begin{align}
    \sum_{i=1}^K\mathbb{E}\left[\norm{\nabla F(\overline{x}_k)}^2\right]\leq &\left(\frac{2(F(\overline{x}_{1})-F^*)}{\eta K}+\frac{2M}{\eta}\right)\frac{1-D}{1-2D}+\left(\frac{2\eta^2L^2\rho(m\sigma^2+\lambda^2\Delta^2)}{m(1-\rho)}+\frac{6\eta^2L^2\zeta^2\rho}{(1-\sqrt{\rho})^2}\right)\frac{1}{1-2D}\notag\\
    \leq&\left(\frac{2(F(\overline{x}_{1})-F^*)}{\eta K}+\frac{2M}{\eta}\right)\frac{1}{1-2D}+\frac{2\eta^2L^2\rho}{1-\sqrt{\rho}}\left(\frac{m\sigma^2+\lambda^2\Delta^2}{m(1+\sqrt{\rho})}+\frac{3\zeta^2}{1-\sqrt{\rho}}\right)\frac{1}{1-2D}
\end{align}
Recall that $\eta L\leq(1-\sqrt{\rho})/4\sqrt{\rho}$, we could know that $\frac{1}{1-2D}\leq 4$. Therefore the bound could be simplified as
\begin{align}
 \sum_{i=1}^K\mathbb{E}\left[\norm{\nabla F(\overline{x}_k)}^2\right]\leq\frac{8(F(\overline{x}_{1})-F^*)}{\eta K}+\frac{8M}{\eta}+\frac{8\eta^2L^2\rho}{1-\sqrt{\rho}}\left(\frac{m\sigma^2+\lambda^2\Delta^2}{m(1+\sqrt{\rho})}+\frac{3\zeta^2}{1-\sqrt{\rho}}\right),
\end{align}
where $M=\frac{\eta^2L\sigma^2}{2m}+\lambda\eta\beta\Delta+\lambda\eta^2L\beta \Delta+\frac{\lambda^2\eta^2L\Delta^2}{2}$.
When $\eta=\lambda=\sqrt{\frac{m}{K}}$,
\begin{align}
    \sum_{i=1}^K\mathbb{E}\left[\norm{\nabla F(\overline{x}_k)}^2\right]\leq& \frac{8(F(\overline{x}_{1})-F^*)}{\sqrt{mK}}+\frac{4 L\sigma^2}{\sqrt{mk}}+8\beta\Delta\sqrt{\frac{m}{K}}+8\lambda L\beta \Delta\sqrt{\frac{m}{K}}+4\lambda^2 L\Delta^2\sqrt{\frac{m}{K}}\notag\\
    &+\frac{8\sqrt{m} L^2\rho}{(1-\sqrt{\rho})\sqrt{K}}\left(\frac{\sigma^2}{1+\sqrt{\rho}}+\frac{\Delta^2}{K(1+\sqrt{\rho})}+\frac{3\zeta^2}{1-\sqrt{\rho}}\right)\notag\\
    =&\mathcal{O}(\frac{1}{\sqrt{mK}})+\mathcal{O}(\sqrt{\frac{m}{K}})+ +\mathcal{O}(\sqrt{\frac{m}{K^3}})
\end{align}

\subsection{Proof for Lemma \ref{lma:1}}\label{supp:lma1}
\begin{proof}[Proof for Lemma \ref{lma:1}.]
Define $A_{q,n}=\prod_{k=q}^n\widetilde{W}^{(k)}$, $b_i^\intercal$ denote the $i$-th row vector of $B$. Thus, we have
$$A_{1,n}=A_{1,n-1}\widetilde{W}^{(n)}.$$
Thus, we have
\begin{align*}
    \mathbb{E}_{\widetilde{W}^{(n)}}\left[\norm{BA_{1,n}(I-J)}_F^2\right]\leq&\sum_{i=1}^d\mathbb{E}_{\widetilde{W}^{(n)}}\left[\norm{b_i^\intercal A_{1,n}(I-J)}^2\right]\notag\\
    =&\sum_{i=1}^db_i^\intercal A_{1,n-1}\mathbb{E}_{\widetilde{W}^{(n)}}\left[\widetilde{W}^{(n)}(I-J)^2\widetilde{W}^{(n)\intercal}\right]A_{1,n-1}^\intercal b_i\notag\\
    =&\sum_{i=1}^db_i^\intercal A_{1,n-1}\mathbb{E}_{\widetilde{W}^{(n)}}\left[\widetilde{W}^{(n)}(I-J)\widetilde{W}^{(n)\intercal}\right]A_{1,n-1}^\intercal b_i
\end{align*}
Let $C=\mathbb{E}_{\widetilde{W}^{(n)}}\left[\widetilde{W}^{(n)}(I-J)\widetilde{W}^{(n)\intercal}\right]$, $v_i=A_{1,n-1}^\intercal b_i$, we have:
\begin{align*}
    \mathbb{E}_{\widetilde{W}^{(n)}}\left[\norm{BA_{1,n}(I-J)}_F^2\right]\leq&\sum_{i=1}^dv_i^\intercal C v_i
    \leq \sigma_{max}(C)\sum_{i=1}^dv_i^\intercal v_i=\rho\norm{BA_{1,n-1}}_F^2
\end{align*}

Similarly, we could have
\begin{align*}
    \mathbb{E}_{\widetilde{W}^{(n-1)}}\left[\norm{BA_{1,n-1}}_F^2\right]\leq&\sum_{i=1}^d\mathbb{E}_{\widetilde{W}^{(n-1)}}\left[\norm{b_i^\intercal A_{1,n-1}}^2\right]\notag\\
    =&\sum_{i=1}^db_i^\intercal A_{1,n-2}\mathbb{E}_{\widetilde{W}^{(n-1)}}\left[\widetilde{W}^{(n-1)}\widetilde{W}^{(n-1)\intercal}\right]A_{1,n-2}^\intercal b_i\notag\\
    \leq&\rho\norm{BA_{1,n-2}}_F^2
\end{align*}

\begin{align*}
        \mathbb{E}_{\widetilde{W}^{(1)},...,\widetilde{W}^{(n)}}\left[\norm{B\left(\prod_{k=1}^n\widetilde{W}^{(k)}\right)(I-J)}_F^2\right]\leq\rho^n\norm{B}_F^2.
    \end{align*}
\end{proof}

\clearpage
\newpage

\section{Dynamic communication graph with three Laplacian matrices.}\label{Appendix: 2}   

\begin{table}[H]
\centering
\setlength{\tabcolsep}{0.4mm}{
\begin{tabular}{lllllllll}
\hline
\multicolumn{1}{c}{\textbf{}}       &\multicolumn{5}{c}{\textbf{TEST ACC}}      \\ \hline
\multicolumn{1}{c}{Node}   &\multicolumn{1}{c}{D-PSGD$^1$}    &\multicolumn{1}{c}{D-PSGD$^2$}    &\multicolumn{1}{c}{D-PSGD$^3$}    &\multicolumn{1}{c}{D-PSGD$^4$}    &\multicolumn{1}{c}{D-PSGD$^5$}        \\ \hline
\multicolumn{1}{c}{0}               &\multicolumn{1}{c}{87.95} &\multicolumn{1}{c}{91.31}      &\multicolumn{1}{c}{91.27}   &\multicolumn{1}{c}{91.28}    &\multicolumn{1}{c}{88.02}           \\ 
\multicolumn{1}{c}{1}               &\multicolumn{1}{c}{92.11} &\multicolumn{1}{c}{91.02}   &\multicolumn{1}{c}{91.12}  &\multicolumn{1}{c}{88.02}    &\multicolumn{1}{c}{88.50} \\ 
\multicolumn{1}{c}{2}               &\multicolumn{1}{c}{92.21} &\multicolumn{1}{c}{91.05}   &\multicolumn{1}{c}{91.17}  &\multicolumn{1}{c}{92.26}    &\multicolumn{1}{c}{92.22} \\ 
\multicolumn{1}{c}{3}               &\multicolumn{1}{c}{92.36} &\multicolumn{1}{c}{91.03}   &\multicolumn{1}{c}{90.02}  &\multicolumn{1}{c}{92.08}    &\multicolumn{1}{c}{92.04} \\ 
\multicolumn{1}{c}{4}               &\multicolumn{1}{c}{87.86} &\multicolumn{1}{c}{91.21}  &\multicolumn{1}{c}{91.04}   &\multicolumn{1}{c}{91.26}    &\multicolumn{1}{c}{87.52} \\ 
\multicolumn{1}{c}{5}               &\multicolumn{1}{c}{92.25} &\multicolumn{1}{c}{91.06}   &\multicolumn{1}{c}{91.12}  &\multicolumn{1}{c}{87.69}    &\multicolumn{1}{c}{89.21} \\ 
\multicolumn{1}{c}{6}               &\multicolumn{1}{c}{92.38} &\multicolumn{1}{c}{91.12}   &\multicolumn{1}{c}{91.08}  &\multicolumn{1}{c}{92.06}    &\multicolumn{1}{c}{91.16}  \\ 
\multicolumn{1}{c}{7}               &\multicolumn{1}{c}{92.32} &\multicolumn{1}{c}{91.95}   &\multicolumn{1}{c}{91.03}  &\multicolumn{1}{c}{87.42}    &\multicolumn{1}{c}{88.08}  \\ \hline
\multicolumn{1}{c}{\textbf{AVG}}               &\multicolumn{1}{c}{91.18} &\multicolumn{1}{c}{91.21}   &\multicolumn{1}{c}{90.98}    &\multicolumn{1}{c}{90.26}    &\multicolumn{1}{c}{89.59}    \\ \hline
\end{tabular}}   
\caption{\label{tab:widgets} Performance of D-PSGD with different toplogy degree on CIFAR-10. D-PSGD$^1$ presents full degree, degree = 13. D-PSGD$^2$ presents $84.6\%$ degree, degree = 11, D-PSGD$^3$ presents $69.2\%$ degree, degree = 9. D-PSGD$^4$ presents $53.8\%$ degree, degree =7, D-PSGD$^5$ presents $38.5\%$ degree, degree = 5.}
\label{Tab:3}
\end{table}

\begin{table}[H]
\centering
\setlength{\tabcolsep}{0.3mm}{
\begin{tabular}{lllllllll}
\hline
\multicolumn{1}{c}{\textbf{}}       &\multicolumn{5}{c}{\textbf{TEST ACC}}      \\ \hline
\multicolumn{1}{c}{Node}   &\multicolumn{1}{c}{ALDSGD$^1$}    &\multicolumn{1}{c}{ALDSGD$^2$}    &\multicolumn{1}{c}{ALDSGD$^3$} &\multicolumn{1}{c}{ALDSGD$^4$}    &\multicolumn{1}{c}{ALDSGD$^5$}        \\ \hline
\multicolumn{1}{c}{0}               &\multicolumn{1}{c}{93.68} &\multicolumn{1}{c}{93.73}     &\multicolumn{1}{c}{93.62}   &\multicolumn{1}{c}{93.47}    &\multicolumn{1}{c}{92.78}           \\ 
\multicolumn{1}{c}{1}               &\multicolumn{1}{c}{93.72} &\multicolumn{1}{c}{93.60}   &\multicolumn{1}{c}{93.51}  &\multicolumn{1}{c}{93.29}    &\multicolumn{1}{c}{92.70} \\ 
\multicolumn{1}{c}{2}               &\multicolumn{1}{c}{93.55} &\multicolumn{1}{c}{93.62}   &\multicolumn{1}{c}{93.59}  &\multicolumn{1}{c}{93.40}    &\multicolumn{1}{c}{92.24} \\ 
\multicolumn{1}{c}{3}               &\multicolumn{1}{c}{93.87} &\multicolumn{1}{c}{93.53}   &\multicolumn{1}{c}{93.28}  &\multicolumn{1}{c}{93.11}    &\multicolumn{1}{c}{91.97} \\ 
\multicolumn{1}{c}{4}               &\multicolumn{1}{c}{93.83} &\multicolumn{1}{c}{93.61}  &\multicolumn{1}{c}{93.22}   &\multicolumn{1}{c}{93.37}    &\multicolumn{1}{c}{92.08} \\ 
\multicolumn{1}{c}{5}               &\multicolumn{1}{c}{93.48} &\multicolumn{1}{c}{93.55}   &\multicolumn{1}{c}{93.78}  &\multicolumn{1}{c}{93.27}    &\multicolumn{1}{c}{92.76} \\ 
\multicolumn{1}{c}{6}               &\multicolumn{1}{c}{93.65} &\multicolumn{1}{c}{93.62}   &\multicolumn{1}{c}{93.87}  &\multicolumn{1}{c}{93.32}    &\multicolumn{1}{c}{91.97}  \\ 
\multicolumn{1}{c}{7}               &\multicolumn{1}{c}{93.62} &\multicolumn{1}{c}{93.48}   &\multicolumn{1}{c}{93.79}  &\multicolumn{1}{c}{93.14}    &\multicolumn{1}{c}{92.07}  \\ \hline
\multicolumn{1}{c}{\textbf{AVG}}               &\multicolumn{1}{c}{93.68} &\multicolumn{1}{c}{93.59}   &\multicolumn{1}{c}{93.58}    &\multicolumn{1}{c}{93.30}    &\multicolumn{1}{c}{92.32}    \\ \hline
\end{tabular}}   
\caption{\label{tab:widgets} Performance of AL-DSGD using D-PSGD as baseline with different toplogy degree on CIFAR-10. AL-DSGD$^1$ presents full degree, degree = 13. AL-DSGD$^2$ presents $84.6\%$ degree, degree = 11, AL-DSGD$^3$ presents $69.2\%$ degree, degree = 9. AL-DSGD$^4$ presents $53.8\%$ degree, degree = 7. AL-DSGD$^5$ presents $38.5\%$ degree, degree = 5.}
\label{Tab:4}
\end{table}

\begin{table}[H]
\centering
\setlength{\tabcolsep}{0.4mm}{
\begin{tabular}{lllllllll}
\hline
\multicolumn{1}{c}{\textbf{}}       &\multicolumn{5}{c}{\textbf{TEST ACC}}      \\ \hline
\multicolumn{1}{c}{Node}   &\multicolumn{1}{c}{MATCHA$^1$}    &\multicolumn{1}{c}{MATCHA$^2$}    &\multicolumn{1}{c}{MATCHA$^3$}    &\multicolumn{1}{c}{MATCHA$^4$}   &\multicolumn{1}{c}{MATCHA$^5$}        \\ \hline
\multicolumn{1}{c}{0}               &\multicolumn{1}{c}{93.78} &\multicolumn{1}{c}{93.48}      &\multicolumn{1}{c}{93.07}   &\multicolumn{1}{c}{92.74}    &\multicolumn{1}{c}{92.78}           \\ 
\multicolumn{1}{c}{1}               &\multicolumn{1}{c}{93.68} &\multicolumn{1}{c}{93.45}   &\multicolumn{1}{c}{93.38}  &\multicolumn{1}{c}{92.93}    &\multicolumn{1}{c}{92.70} \\ 
\multicolumn{1}{c}{2}               &\multicolumn{1}{c}{92.76} &\multicolumn{1}{c}{93.58}   &\multicolumn{1}{c}{93.31}  &\multicolumn{1}{c}{92.96}    &\multicolumn{1}{c}{92.34} \\ 
\multicolumn{1}{c}{3}               &\multicolumn{1}{c}{93.91} &\multicolumn{1}{c}{93.46}   &\multicolumn{1}{c}{93.18}  &\multicolumn{1}{c}{92.96}    &\multicolumn{1}{c}{92.04} \\ 
\multicolumn{1}{c}{4}               &\multicolumn{1}{c}{93.68} &\multicolumn{1}{c}{93.52}  &\multicolumn{1}{c}{93.22}  &\multicolumn{1}{c}{92.83}    &\multicolumn{1}{c}{88.28} \\ 
\multicolumn{1}{c}{5}               &\multicolumn{1}{c}{93.81} &\multicolumn{1}{c}{93.55}   &\multicolumn{1}{c}{93.18}  &\multicolumn{1}{c}{92.72}    &\multicolumn{1}{c}{92.28} \\ 
\multicolumn{1}{c}{6}               &\multicolumn{1}{c}{93.72} &\multicolumn{1}{c}{93.49}   &\multicolumn{1}{c}{93.21}  &\multicolumn{1}{c}{92.97}    &\multicolumn{1}{c}{92.16}  \\ 
\multicolumn{1}{c}{7}               &\multicolumn{1}{c}{93.82} &\multicolumn{1}{c}{93.58}   &\multicolumn{1}{c}{93.34}  &\multicolumn{1}{c}{92.80}    &\multicolumn{1}{c}{87.67}  \\ \hline
\multicolumn{1}{c}{\textbf{AVG}}               &\multicolumn{1}{c}{93.65} &\multicolumn{1}{c}{93.51}   &\multicolumn{1}{c}{93.24}    &\multicolumn{1}{c}{92.86}    &\multicolumn{1}{c}{91.14}    \\ \hline
\end{tabular}}   
\caption{\label{tab:widgets}Performance MATCHA with different toplogy degree on CIFAR-10. MATCHA$^1$ presents full degree, degree = 13. MATCHA$^2$ presents $84.6\%$ degree, degree = 11, MATCHA$^3$ presents $69.2\%$ degree, degree = 9. MATCHA$^4$ presents $53.8\%$ degree, degree = 7. MATCHA$^5$ presents $38.5\%$ degree, degree = 5.}
\label{Tab:5}
\end{table}

\begin{table}[H]
\centering
\setlength{\tabcolsep}{0.4mm}{
\begin{tabular}{lllllllll}
\hline
\multicolumn{1}{c}{\textbf{}}       &\multicolumn{5}{c}{\textbf{TEST ACC}}      \\ \hline
\multicolumn{1}{c}{Node}   &\multicolumn{1}{c}{ALDSGD$^1$}    &\multicolumn{1}{c}{ALDSGD$^2$}    &\multicolumn{1}{c}{ALDSGD$^3$}    &\multicolumn{1}{c}{ALDSGD$^4$}    &\multicolumn{1}{c}{ALDSGD$^5$}        \\ \hline
\multicolumn{1}{c}{0}               &\multicolumn{1}{c}{93.87} &\multicolumn{1}{c}{93.36}      &\multicolumn{1}{c}{93.46}  &\multicolumn{1}{c}{93.39}    &\multicolumn{1}{c}{92.66}           \\ 
\multicolumn{1}{c}{1}               &\multicolumn{1}{c}{94.42} &\multicolumn{1}{c}{93.43}   &\multicolumn{1}{c}{93.43}  &\multicolumn{1}{c}{93.29}    &\multicolumn{1}{c}{92.85} \\ 
\multicolumn{1}{c}{2}               &\multicolumn{1}{c}{93.69} &\multicolumn{1}{c}{93.27}   &\multicolumn{1}{c}{93.46}  &\multicolumn{1}{c}{93.32}    &\multicolumn{1}{c}{92.71} \\ 
\multicolumn{1}{c}{3}               &\multicolumn{1}{c}{93.86} &\multicolumn{1}{c}{93.29}   &\multicolumn{1}{c}{93.46}  &\multicolumn{1}{c}{93.28}    &\multicolumn{1}{c}{92.85} \\ 
\multicolumn{1}{c}{4}               &\multicolumn{1}{c}{93.94} &\multicolumn{1}{c}{93.49}  &\multicolumn{1}{c}{93.43}   &\multicolumn{1}{c}{93.35}    &\multicolumn{1}{c}{92.68} \\ 
\multicolumn{1}{c}{5}               &\multicolumn{1}{c}{93.88} &\multicolumn{1}{c}{93.18}   &\multicolumn{1}{c}{93.50}  &\multicolumn{1}{c}{93.18}    &\multicolumn{1}{c}{92.87} \\ 
\multicolumn{1}{c}{6}               &\multicolumn{1}{c}{93.89} &\multicolumn{1}{c}{93.33}   &\multicolumn{1}{c}{93.50}  &\multicolumn{1}{c}{93.31}    &\multicolumn{1}{c}{92.65}  \\ 
\multicolumn{1}{c}{7}               &\multicolumn{1}{c}{93.98} &\multicolumn{1}{c}{93.29}   &\multicolumn{1}{c}{93.64}  &\multicolumn{1}{c}{93.28}    &\multicolumn{1}{c}{92.70}  \\ \hline
\multicolumn{1}{c}{\textbf{AVG}}               &\multicolumn{1}{c}{93.94} &\multicolumn{1}{c}{93.33}   &\multicolumn{1}{c}{93.49}    &\multicolumn{1}{c}{93.30}    &\multicolumn{1}{c}{92.75}    \\ \hline
\end{tabular}}   
\caption{\label{tab:widgets} Performance of AL-DSGD using MATCHA as baseline with different toplogy degree on CIFAR-10. AL-DSGD$^1$ presents full degree, degree = 13. AL-DSGD$^2$ presents $84.6\%$ degree, degree = 11, AL-DSGD$^3$ presents $69.2\%$ degree, degree = 9. AL-DSGD$^4$ presents $53.8\%$ degree, degree =7, AL-DSGD$^5$ presents $38.5\%$ degree, degree = 5.}
\label{Tab:6}
\end{table}


From table ~\ref{Tab:3}, we can tell limit degree D-PSGD convergences slower and achieve worse final model comparing with full degree D-PSGD. By applying dynamic communication graphs, AL-DSGD is more robust to sparse topology under limited communication environment. From table~\ref{Tab:4}, we can tell the performance of final model remains the same when we reduce degree to 69\%. The performance of AL-DSGD with 38\% of degree is even better than the performance of D-PSGD with 84\% of degree. We can get the same conclusion from table~\ref{Tab:5} and table~\ref{Tab:6} when applying AL-DSGD to MATCHA baseline.



\newpage
\section{Dynamic communication graph with two Laplacian matrices.}\label{Appendix: 3}

\begin{figure}[H]
    \centering
    \includegraphics[width=60mm]{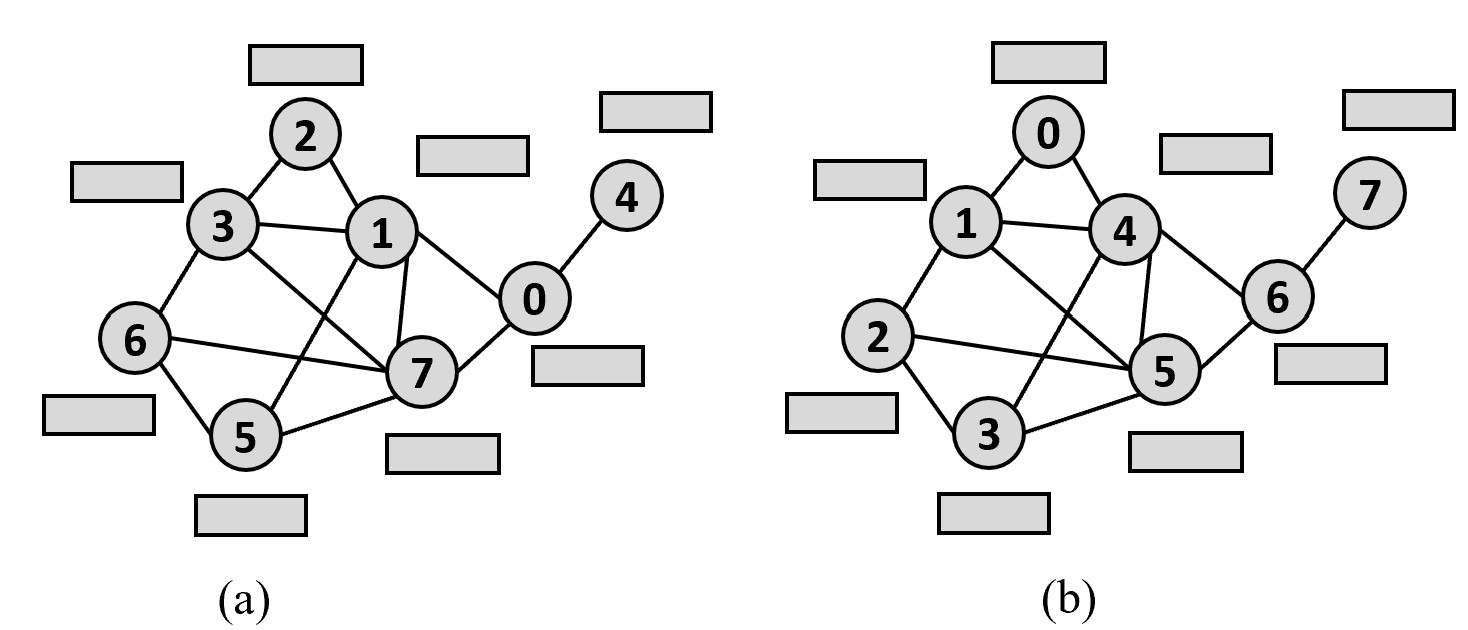}
    \centering
    \caption{Rotate the location of worker between (a) and (b).}
    \label{Fig: 8}
\end{figure}

\begin{table}[H]
\centering
\setlength{\tabcolsep}{0.3mm}{
\begin{tabular}{lllllllll}
\hline
\multicolumn{1}{c}{\textbf{}}       &\multicolumn{5}{c}{\textbf{TEST ACC}}      \\ \hline
\multicolumn{1}{c}{\textbf{Node}}   &\multicolumn{1}{c}{\textbf{ALDSGD$^1$}}    &\multicolumn{1}{c}{\textbf{ALDSGD$^2$}}    &\multicolumn{1}{c}{\textbf{ALDSGD$^3$}} &\multicolumn{1}{c}{\textbf{ALDSGD$^4$}}    &\multicolumn{1}{c}{\textbf{ALDSGD$^5$}}        \\ \hline
\multicolumn{1}{c}{0}               &\multicolumn{1}{c}{93.40} &\multicolumn{1}{c}{93.46}      &\multicolumn{1}{c}{93.60}  &\multicolumn{1}{c}{92.61}    &\multicolumn{1}{c}{91.68}           \\ 
\multicolumn{1}{c}{1}               &\multicolumn{1}{c}{93.28} &\multicolumn{1}{c}{93.42}   &\multicolumn{1}{c}{93.46}  &\multicolumn{1}{c}{92.46}    &\multicolumn{1}{c}{91.26}  \\ 
\multicolumn{1}{c}{2}               &\multicolumn{1}{c}{93.42} &\multicolumn{1}{c}{93.46}   &\multicolumn{1}{c}{93.33}  &\multicolumn{1}{c}{92.44}    &\multicolumn{1}{c}{91.95}  \\ 
\multicolumn{1}{c}{3}               &\multicolumn{1}{c}{93.44} &\multicolumn{1}{c}{93.39}   &\multicolumn{1}{c}{93.39}  &\multicolumn{1}{c}{92.34}    &\multicolumn{1}{c}{91.96}  \\ 
\multicolumn{1}{c}{4}               &\multicolumn{1}{c}{93.38} &\multicolumn{1}{c}{93.46}  &\multicolumn{1}{c}{93.46}  &\multicolumn{1}{c}{92.39}    &\multicolumn{1}{c}{91.70}  \\ 
\multicolumn{1}{c}{5}               &\multicolumn{1}{c}{93.24} &\multicolumn{1}{c}{93.31}   &\multicolumn{1}{c}{93.46}  &\multicolumn{1}{c}{92.56}    &\multicolumn{1}{c}{91.84}  \\ 
\multicolumn{1}{c}{6}               &\multicolumn{1}{c}{93.11} &\multicolumn{1}{c}{93.41}   &\multicolumn{1}{c}{93.34}  &\multicolumn{1}{c}{92.47}    &\multicolumn{1}{c}{91.76}   \\ 
\multicolumn{1}{c}{7}               &\multicolumn{1}{c}{93.32} &\multicolumn{1}{c}{93.54}   &\multicolumn{1}{c}{93.35}  &\multicolumn{1}{c}{92.43}    &\multicolumn{1}{c}{91.74} \\ \hline
\multicolumn{1}{c}{\textbf{AVG}}               &\multicolumn{1}{c}{93.32} &\multicolumn{1}{c}{93.43}   &\multicolumn{1}{c}{93.47}    &\multicolumn{1}{c}{92.46}    &\multicolumn{1}{c}{91.73}     \\ \hline
\end{tabular}}   
\caption{\label{tab:widgets} Performance of AL-DSGD using D-PSGD as baseline with different toplogy degree on CIFAR-10. AL-DSGD$^1$ presents full degree, degree = 13. AL-DSGD$^2$ presents $84.6\%$ degree, degree = 11, AL-DSGD$^3$ presents $69.2\%$ degree, degree = 9. AL-DSGD$^4$ presents $53.8\%$ degree, degree = 7. AL-DSGD$^5$ presents $38.5\%$ degree, degree = 5.}
\label{Tab:9}
\end{table}

\begin{table}[H]
\centering
\setlength{\tabcolsep}{0.4mm}{
\begin{tabular}{lllllllll}
\hline
\multicolumn{1}{c}{\textbf{}}       &\multicolumn{5}{c}{\textbf{TEST ACC}}      \\ \hline
\multicolumn{1}{c}{\textbf{Node}}   &\multicolumn{1}{c}{\textbf{ALDSGD$^1$}}    &\multicolumn{1}{c}{\textbf{ALDSGD$^2$}}    &\multicolumn{1}{c}{\textbf{ALDSGD$^3$}}    &\multicolumn{1}{c}{\textbf{ALDSGD$^4$}}    &\multicolumn{1}{c}{\textbf{ALDSGD$^5$}}        \\ \hline
\multicolumn{1}{c}{0}               &\multicolumn{1}{c}{93.85} &\multicolumn{1}{c}{93.26}      &\multicolumn{1}{c}{93.42}   &\multicolumn{1}{c}{92.94}    &\multicolumn{1}{c}{92.28}           \\ 
\multicolumn{1}{c}{1}               &\multicolumn{1}{c}{94.21} &\multicolumn{1}{c}{93.23}   &\multicolumn{1}{c}{93.39}  &\multicolumn{1}{c}{93.01}    &\multicolumn{1}{c}{92.27} \\ 
\multicolumn{1}{c}{2}               &\multicolumn{1}{c}{93.76} &\multicolumn{1}{c}{93.38}   &\multicolumn{1}{c}{93.42}  &\multicolumn{1}{c}{93.00}    &\multicolumn{1}{c}{92.28} \\ 
\multicolumn{1}{c}{3}               &\multicolumn{1}{c}{93.78} &\multicolumn{1}{c}{93.26}   &\multicolumn{1}{c}{93.39}  &\multicolumn{1}{c}{93.16}    &\multicolumn{1}{c}{92.31} \\ 
\multicolumn{1}{c}{4}               &\multicolumn{1}{c}{93.59} &\multicolumn{1}{c}{93.29}  &\multicolumn{1}{c}{93.30}   &\multicolumn{1}{c}{93.14}    &\multicolumn{1}{c}{92.28} \\ 
\multicolumn{1}{c}{5}               &\multicolumn{1}{c}{93.86} &\multicolumn{1}{c}{93.16}   &\multicolumn{1}{c}{93.35}  &\multicolumn{1}{c}{93.05}    &\multicolumn{1}{c}{92.27} \\ 
\multicolumn{1}{c}{6}               &\multicolumn{1}{c}{93.97} &\multicolumn{1}{c}{93.30}   &\multicolumn{1}{c}{93.41}  &\multicolumn{1}{c}{93.31}    &\multicolumn{1}{c}{92.29}  \\ 
\multicolumn{1}{c}{7}               &\multicolumn{1}{c}{93.86} &\multicolumn{1}{c}{93.22}   &\multicolumn{1}{c}{93.26}  &\multicolumn{1}{c}{93.09}    &\multicolumn{1}{c}{92.30}  \\ \hline
\multicolumn{1}{c}{\textbf{AVG}}               &\multicolumn{1}{c}{93.86} &\multicolumn{1}{c}{93.26}   &\multicolumn{1}{c}{93.36}    &\multicolumn{1}{c}{93.08}    &\multicolumn{1}{c}{92.29}    \\ \hline
\end{tabular}}   
\caption{\label{tab:widgets} Performance of AL-DSGD using MATCHA as baseline with different toplogy degree on CIFAR-10. D-PSGD$^1$ presents full degree, degree = 13. D-PSGD$^2$ presents $84.6\%$ degree, degree = 11, D-PSGD$^3$ presents $69.2\%$ degree, degree = 9. D-PSGD$^4$ presents $53.8\%$ degree, degree =7, D-PSGD$^5$ presents $38.5\%$ degree, degree = 5.}
\label{Tab:8}
\end{table}

We apply dynamic communication graph with two Laplacian matrices. The dynamic communication graph can be found in figure~\ref{Fig: 8}. The results of AL-DSGD using D-PSGD and MATCHA as baselines can be found in table~\ref{Tab:9} and table~\ref{Tab:8} respectively. Compare them with table~\ref{Tab:3} and table~\ref{Tab:5}, we can clearly see the improvement when using the dynamic communication graph with two Laplacian matrices.

\newpage
\section{Performance of MATCHA Baseline and MATCHA Based AL-DSGD}\label{Appendix: 4}

\begin{table}[H]
\centering
\vspace{-0.05in}
\setlength{\tabcolsep}{1.2mm}{
\begin{tabular}{lllllllll}
\hline
\multicolumn{1}{c}{\textbf{}}       &\multicolumn{2}{c|}{\textbf{CIFAR-10/ResNet-50}}      &\multicolumn{2}{c}{\textbf{CIFAR-100/WideResNet}}      \\ 
\hline
\multicolumn{1}{c}{\textbf{Node}}   &\multicolumn{1}{c}{\textbf{MATCHA}}    &\multicolumn{1}{c|}
{\textbf{AL-DSGD}}    &\multicolumn{1}{c}{\textbf{MATCHA}}    &\multicolumn{1}{c}{\textbf{AL-DSGD}}    \\ \hline
\multicolumn{1}{c}{0}               &\multicolumn{1}{c}{93.78} &\multicolumn{1}{c|}{93.87}      &\multicolumn{1}{c}{76.90} &\multicolumn{1}{c}{76.85}    \\ 
\multicolumn{1}{c}{1}               &\multicolumn{1}{c}{93.68} &\multicolumn{1}{c|}{93.42}   &\multicolumn{1}{c}{76.94} &\multicolumn{1}{c}{77.15}\\ 
\multicolumn{1}{c}{2}               &\multicolumn{1}{c}{92.76} &\multicolumn{1}{c|}{93.69}   &\multicolumn{1}{c}{77.03} &\multicolumn{1}{c}{77.29}\\ 
\multicolumn{1}{c}{3}               &\multicolumn{1}{c}{92.91} &\multicolumn{1}{c|}{93.86}   &\multicolumn{1}{c}{77.07} &\multicolumn{1}{c}{77.02}\\ 
\multicolumn{1}{c}{4}               &\multicolumn{1}{c}{93.68} &\multicolumn{1}{c|}{93.94}  &\multicolumn{1}{c}{77.02} &\multicolumn{1}{c}{77.23}\\ 
\multicolumn{1}{c}{5}               &\multicolumn{1}{c}{93.81} &\multicolumn{1}{c|}{93.88}   &\multicolumn{1}{c}{\textbf{74.62}$\downarrow$} &\multicolumn{1}{c}{77.43}\\ 
\multicolumn{1}{c}{6}               &\multicolumn{1}{c}{93.72} &\multicolumn{1}{c|}{93.89}   &\multicolumn{1}{c}{76.59} &\multicolumn{1}{c}{77.30}                 \\ 
\multicolumn{1}{c}{7}               &\multicolumn{1}{c}{93.82} &\multicolumn{1}{c|}{93.98}   &\multicolumn{1}{c}{76.77} &\multicolumn{1}{c}{77.19}                 \\ \hline
\multicolumn{1}{c}{\textbf{AVG}}    &\multicolumn{1}{c}{93.65} &\multicolumn{1}{c|}{\textbf{93.94}}   &\multicolumn{1}{c}{76.62} &\multicolumn{1}{c}{\textbf{77.18}}    \\ \hline
\end{tabular}} 
\caption{\label{tab:widgets} Test accuracy obtained with MATCHA and MATCHA-based AL-DSGD for ResNet-50 model trained on CIFAR-10 and WideResNet model trained on CIFAR-100.}
\label{Tab:10}
\end{table}

\begin{figure}[H]
    \centering
    \includegraphics[width=140mm]{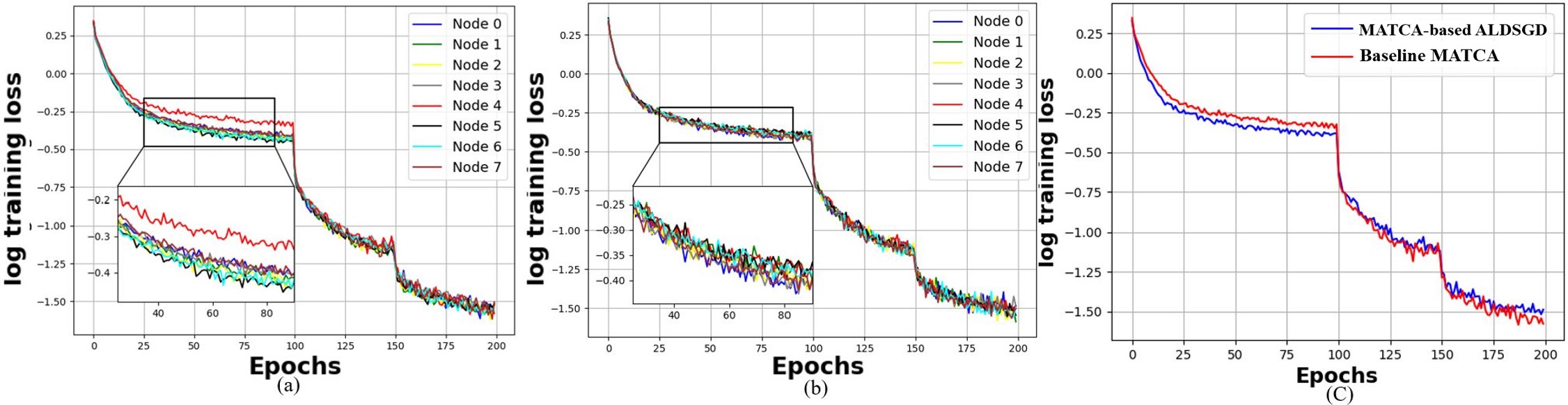}
    \centering
    \caption{ResNet-50 on CIFAR-10. (a):MATCHA baseline (b):AL-DSGD (c): Comparison between worst performance workers in a and b}
    \label{Fig: 7}
\end{figure}

\begin{figure}[H]
    \centering
    \includegraphics[width=140mm]{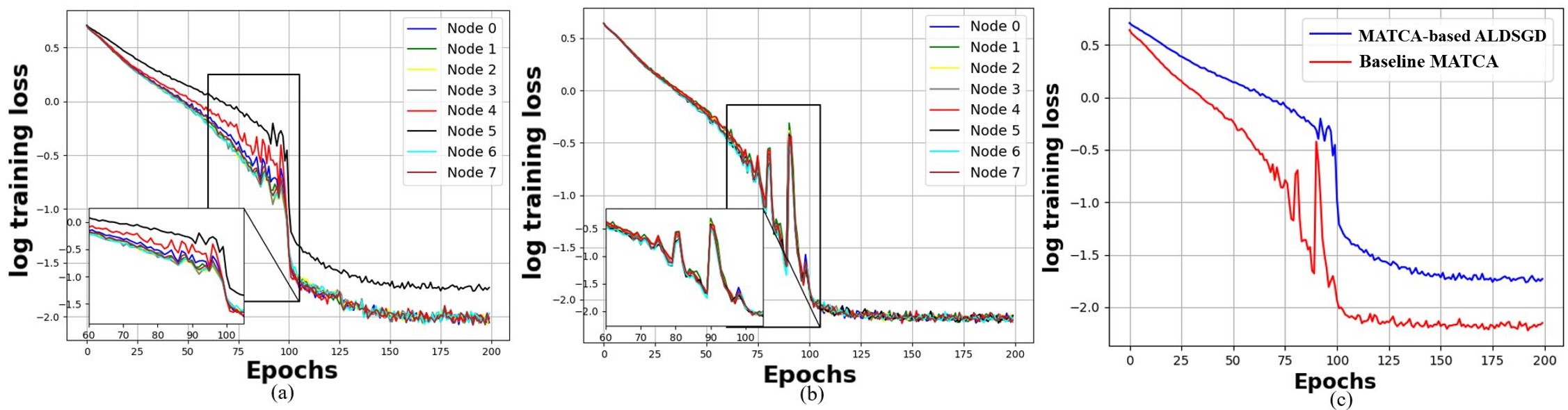}
    \centering
    \vspace{-0.05in}
    \caption{WideResNet on CIFAR-100. (a):MATCHA baseline (b):AL-DSGD (c): Comparison between worst performance workers in a and b}
    \label{Fig: 13}
\end{figure}

\begin{figure}[H]
    \centering
    \includegraphics[width=130mm]{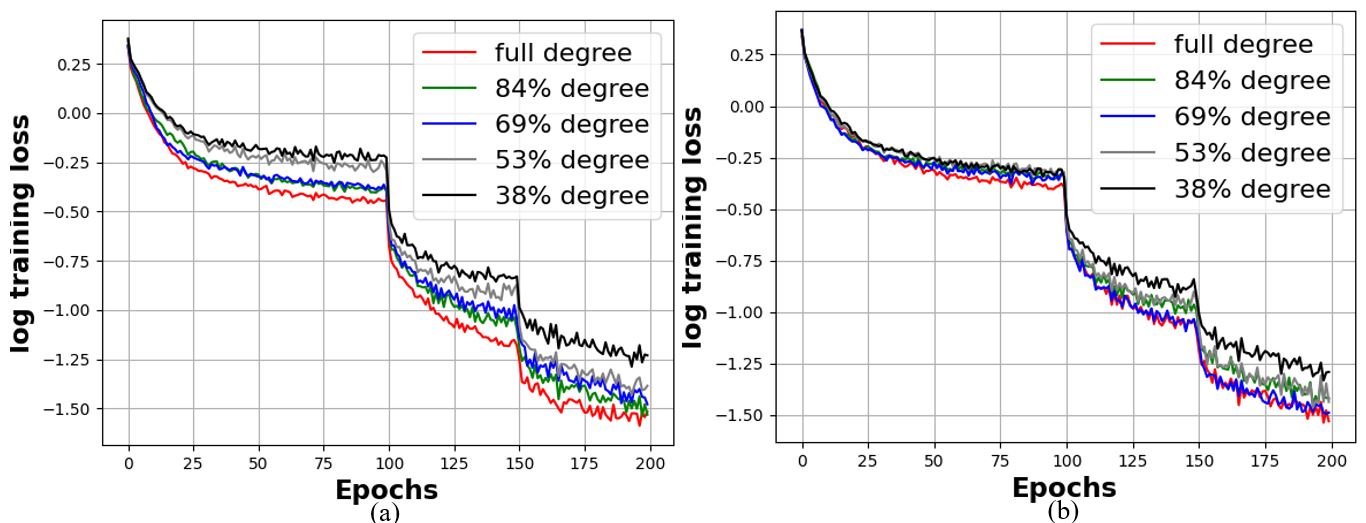}
    \centering
    \caption{ResNet-50 on CIFAR-10. (a):Performance of MATCHA with different topology degrees (b): Performance of MATCHA based AL-DSGD with different topology degrees.}
    \label{Fig: 18}
\end{figure}

The results obtained when training the models with MATCHA and MATCHA-based AL-DSGD (i.e., AL-DSGD implemented on the top of D-PSGD) are presented in Table~\ref{Tab:10}. In The comparison of MATCHA and MATCHA based AL-DSGD can be found in figure~\ref{Fig: 7} and figure~\ref{Fig: 13}. From both figures, we can find MATCHA based AL-DSGD achieves more stable and faster converge comparing with MATCHA baseline. In tabel~\ref{Tab:10}, we incude the result of MATCHA and AL-DSGD.